\documentclass[11pt]{article}

\usepackage{amsmath,amssymb,amsthm,hyperref}
\usepackage[numbers]{natbib}
\usepackage{algorithm, algorithmic}
\usepackage{enumerate}
\usepackage{fullpage}
\usepackage{bandits}
\usepackage{color}
\usepackage{graphicx}

\newcommand{\kl}[2]{\mathop{\rm KL}\left({#1} |\!| {#2}\right)}

\newcommand{\comment}[1]{}

\definecolor{brick}{rgb}{.5,0,.1}

\begin{document}

\begin{center}
{\LARGE{{\bf{Oracle inequalities for computationally adaptive model
        selection}}}}

\vspace*{.2in}

\begin{tabular}{ccc}
  Alekh Agarwal$^\dagger$
&
  Peter L.\ Bartlett$^{\star, \dagger, \ddag}$
&
  John C.\ Duchi$^\dagger$ \\
  \texttt{alekh@eecs.berkeley.edu} & 
  \texttt{bartlett@stat.berkeley.edu} &
  \texttt{jduchi@eecs.berkeley.edu}
\end{tabular}

\vspace*{.2in}

\begin{tabular}{cc}
  Department of Statistics$^\star$, and & Department of Mathematical Sciences$^\ddag$\\
  Department of EECS$^\dagger$, & Queensland University of Technology\\
  University of California, Berkeley, CA USA &   Brisbane, Australia 
\end{tabular}

\vspace*{.2in}

\today

\vspace*{.2in}

\begin{abstract}
  We analyze general model selection procedures using penalized
  empirical loss minimization under computational constraints. While
  classical model selection approaches do not consider computational
  aspects of performing model selection, we argue that any practical
  model selection procedure must not only trade off estimation and
  approximation error, but also the computational
  effort required to compute empirical minimizers for different
  function classes. We provide a framework for analyzing such
  problems, and we give algorithms for model selection under a
  computational budget. These algorithms satisfy oracle inequalities
  that show that the risk of the selected model is not much worse than
  if we had devoted all of our computational budget to the optimal
  function class.
\end{abstract}

\end{center}

\section{Introduction}

In decision-theoretic statistical settings, one receives samples
$\{\samp_1, \ldots, \samp_n\} \subseteq \samplespace$ drawn
i.i.d.\ from some unknown distribution $P$ over a sample space
$\samplespace$, and given a loss function $\loss$, seeks a function
$f$ to minimize the risk
\begin{equation}
  \risk(f) \defeq \E[\loss(\samp, f)].
  \label{eqn:risk-def}
\end{equation}
Since $\risk(f)$ is unknown, the typical approach is to compute
estimates based on the empirical risk, $\emprisk[n](f) \defeq \ninv
\sum_{i=1}^n \loss(\samp_i, f)$, over a function class $\F$. Through
this, one seeks a function $f_n$ with a risk close to the Bayes risk,
the minimal risk over all measurable functions, which is $\risklb
\defeq \inf_f \risk(f)$. There is a natural tradeoff based on the
class $\F$ one chooses, since
\begin{equation*}
  \risk(f_n) - \risklb = \left(\risk(f_n) - \inf_{f \in \F} \risk(f)\right)
  + \left(\inf_{f \in \F} \risk(f) - \risklb\right),
\end{equation*}
which decomposes the excess risk of $f_n$ into estimation error (left) and
approximation error (right).


A common approach to addressing this tradeoff is to express $\F$ as a
union of classes

\begin{equation}
\F = \bigcup_{j \geq 1} \F_j.
\label{eqn:funclassdecomp}
\end{equation}
The {\em model selection problem} is to choose a class $\F_i$ and a
function $f\in\F_i$ that give the best tradeoff between estimation
error and approximation error. A standard approach to the model
selection problem is the now classical idea of \emph{complexity
  regularization}, which arose out of early works by \citet{Mallows73}
and \citet{Akaike74}. The complexity regularization approach balances
two competing objectives: the minimum empirical risk of a model class
$\F_i$ (approximation error) and a complexity penalty (to control
estimation error) for the class. Different choices of the complexity
penalty give rise to different model selection criteria and algorithms
(for example, see the lecture notes by \citet{Massart03} and the
references therein).  The complexity regularization approach uses
penalties $\pen_i : \N \rightarrow \R_+$ associated with each class
$\F_i$ to perform model selection, where $\pen_i(n)$ is a complexity
penalty for class $i$ when $n$ samples are available; usually the
functions $\pen_i$ decrease to zero in $n$ and increase in the index
$i$.  The actual algorithm is as follows: for each $i$, choose
\begin{equation}
  \femp{i} \in \argmin_{f \in \F_i} \emprisk[n](f) ~~~
  \mbox{and~select} ~~~ \widetilde{f}_n = \argmin_{i = 1, 2, \ldots}
  \left\{\emprisk[n](\femp{i}) + \pen_i(n) \right\}
  \label{eqn:model-selection}
\end{equation}
as the output of the model selection procedure, where
$\emprisk[n]$ denotes the $n$-sample empirical risk.
Results of several
authors~\cite{BartlettBoLu02, LugosiWe04, Massart03} show that with
appropriate penalties $\pen_i$ and given a dataset
of size $n$, the output $\widetilde{f}_n$ of the procedure roughly satisfies
\begin{equation}
  \E \risk(\widetilde{f}_n) - \risklb
  \le \min_i \left[ \inf_{f \in \F_i} R(f) - \risklb + \pen_i(n)
    \right] + \order\left(\frac{1}{\sqrt{n}}\right).
  \label{eqn:complexity-regularization-guarantee}
\end{equation}
Several approaches to complexity regularization are possible, and an
incomplete bibliography includes the papers~\cite{VapnikCh74, GemanHw82,
  Rissanen83, Barron91, BartlettBoLu02,LugosiWe04}.

Oracle inequalities of the
form~\eqref{eqn:complexity-regularization-guarantee} show that, for a
given sample size, complexity regularization procedures trade off the
approximation and estimation errors, often optimally~\cite{Massart03}.
A drawback of the above approaches is that in order to provide
guarantees on the result of the model selection procedure, one needs
to be able to optimize over each model in the hierarchy (that is,
compute the estimates $\femp{i}$ for each $i$). This is reasonable
when the sample size $n$ is the key limitation, and it is
computationally feasible when $n$ is small and the samples $\samp$ are
low-dimensional. However, the cost of fitting a large number of model
classes on a large, high-dimensional dataset can be
prohibitive; such data is common in modern statistical settings.  In
such cases, it is the computational resources---rather than the sample
size---that form the key inferential bottleneck.  In this paper, we
consider model selection from this computational perspective, viewing
the amount of computation, rather than the sample size, as the
quantity whose effects on estimation we must understand. Specifically,
we study model selection methods that work within a given
computational budget.

An interesting and difficult aspect of the problem that we must
address is the interaction between model class complexity and
computation time. It is natural to assume that for a fixed sample
size, it is more expensive to estimate a model from a complex class
than a simple class. Put inversely, given a computational bound, a
simple model class can fit a model to a much larger sample size than a
rich model class. So any strategy for model selection under a
computational constraint should trade off two criteria: (i) the
relative training cost of different model classes, which allows
simpler classes to receive far more data (thus making them resilient
to overfitting), and (ii) lower approximation error in the more
complex model classes.

In addressing these computational and statistical issues, this paper
makes two main contributions. First, we propose a novel computational
perspective on the model selection problem, which we believe should be
a natural consideration in statistical learning problems. Secondly,
within this framework, we provide algorithms for model selection in
many different scenarios, and provide oracle inequalities on their
estimates under different assumptions. Our first two results address
the case where we have a model hierarchy that is ordered by inclusion,
that is, $\F_1 \subseteq \F_2 \subseteq \F_3 \subseteq \ldots$. The
first result provides an inequality that is competitive with an oracle
knowing the optimal class, incurring at most an additional logarithmic
penalty in the computational budget. The second result extends our
approach to obtaining faster rates for model selection under
conditions that guarantee sharper concentration results for empirical
risk minimization procedures; oracle inequalities under these
conditions, but without computational constraints, have been obtained,
for example, by~\citet{Bartlett08} and \citet{Koltchinskii06a}.
Both of our results refine existing
complexity-regularized risk minimization techniques by a careful
consideration of the structure of the problem. Our third result
applies to model classes that do not necessarily share any common
structure. Here we present a novel algorithm---exploiting techniques
for multi-armed bandit problems---that uses confidence bounds based on
concentration inequalities to select a good model under a given
computational budget. We also prove a minimax optimal oracle
inequality on the performance of the selected model. All of our
algorithms are computationally simple and efficient.

The remainder of this paper is organized as follows. We begin in
Section~\ref{sec:nesting} by formalizing our setting for a nested
hierarchy of models, providing an estimator and oracle inequalities
for the model selection problem. In Section~\ref{sec:nesting-fast}, we
refine our estimator and its analysis to obtain fast rates for model
selection under some additional reasonable (standard) conditions.  We
study the setting of unstructured model collections in
Section~\ref{sec:bandits}. Detailed technical arguments and various
auxilliary results needed to establish our main theorems and
corollaries can be found in the appendices.


\section{Model selection over nested hierarchies}
\label{sec:nesting} 

In many practical scenarios, the family of models with which one works
has some structure. One of the most common model selection settings
has the model classes $\F_i$ ordered by inclusion with increasing
complexity (e.g.~\cite{BartlettBoLu02}). In this section, we study
such model selection problems; we begin by formally stating our
assumptions and giving a few natural examples, proceeding thereafter
to oracle inequalities for a computationally efficient model selection
procedure.

\subsection{Assumptions}

Our first main assumption is a natural inclusion assumption, which is
perhaps the most common assumption in prior work on model selection
(e.g.~\cite{BartlettBoLu02,LugosiWe04}):
\begin{assumption}
  \label{assumption:inclusion}
  The function classes $\F_i$ are ordered by inclusion:
  \begin{equation}
    \label{eqn:inclusion}
    \F_1 \subseteq \F_2 \subseteq \F_3 \subseteq \ldots 
  \end{equation}
\end{assumption}
\noindent
We provide two examples of such problems in the next section.
In addition to the inclusion assumption, we make a few assumptions on
the computational aspects of the problem.  Most algorithms used in the
framework of complexity regularization rely on the computation of
estimators of the form
\begin{equation}
  \what{f}_i = \argmin_{f \in \F_i} \emprisk[n](f),
  \label{eqn:nobudget}
\end{equation}
either exactly or approximately, for each class $i$. Since the model
classes are ordered by inclusion, it is natural to assume that the
computational cost of computing an empirical risk minimizer from
$\F_i$ is higher than that for a class $\F_j$ when $i > j$. Said
differently, given a fixed computational budget $T$, it may be
impossible to use as many samples to compute an estimator from $\F_i$
as it is to compute an estimator from $\F_j$ (again, when $i > j$).
We formalize this in the next assumption, which is stated in terms of
an (arbitrary) algorithm $\alg$ that selects functions $f \in \F_i$
for each index $i$ based on a set of $n_i$ samples.
\begin{assumption}
  Given a computational budget $T$, there is a sequence
  $\{n_i(T)\}_i \subset \N$ such that
  \begin{enumerate}[(a)]
  \item $n_i(T) > n_j(T)$ for $i < j$.
  \item The complexity penalties $\pen_i$ satisfy $\pen_i(n_i(T)) <
    \pen_j(n_j(T))$ for $i < j$.
  \item
    For each class $\F_i$, the computational cost of using the
    algorithm $\alg$ with $n_i(T)$ samples is $T$. That is, estimation
    within class $\F_i$ using $n_i(T)$ samples has the same
    computational complexity for each $i$.
  \item For all $i$, the output $\falg{i}{T}$ of the algorithm $\alg$, given a
    computational budget $T$, satisfies
    \begin{equation*}
      \emprisk[n_i(T)](\falg{i}{T}) - \inf_{f \in \F_i}
      \emprisk[n_i(T)](f) \leq \pen_i(n_i(T)). 
    \end{equation*}
  \item As $i \uparrow \infty$, $\pen_i(n) \rightarrow \infty$ for any
    fixed $n$.
  \end{enumerate}
  \label{assumption:budget}
\end{assumption}

The first two assumptions formalize a natural notion of computational
budget in the context of our model selection problem: given equal
computation time, a simpler model can be fit using a larger number of
samples than a complex model. Assumption~\ref{assumption:budget}(c)
says that the number of samples $n_i(T)$ is chosen to roughly equate
the computational complexity of estimation within each class.
Assumption~\ref{assumption:budget}(d) simply states that we compute
approximate empirical minimizers for each class $\F_i$.  Our choice of
the accuracy of computation to be $\pen_i$ in part (d) is done mainly
for notational convenience in the statements of our results; one could
use an alternate constant or function and achieve similar results.
Finally part (e) rules out degenerate cases where the penalty function
asymptotes to a finite upper bound, and this assumption is required
for our estimator to be well-defined for infinite model
hierarchies. In the sequel, we use the shorthand $\pen_i(T)$ to denote
$\pen_i(n_i(T))$ when the number of samples $n_i(T)$ is clear from
context.

Certainly many choices are possible for the penalty functions
$\pen_i$, and work studying appropriate penalties is classical (see
e.g.~\cite{Akaike74,Mallows73}).  Our focus in this paper is on
complexity estimates derived from concentration inequalities, which
have been extensively studied by a number of
researchers~\cite{BartlettBoLu02, Massart03,BarronBiMa99,Bartlett08,
  Koltchinskii06a}. Such complexity estimates are convenient since
they ensure that the penalized empirical risk bounds the true risk
with high probability. Formally, we have
\begin{assumption}
  \label{assumption:uniformity}
  For all $\epsilon > 0$ and for each $i$, there are constants
  $\const_1, \const_2 > 0$ such that for any budget $T$ the output
  $\falg{i}{T} \in \F_i$ satisfies,
  \begin{equation}
    \P\left(|\emprisk[n_i(T)](\falg{i}{T}) - \risk(\falg{i}{T})| >
    \pen_i(T) + \const_2 \eps\right) \leq \const_1 \exp(-4 n_i(T) \eps^2).
    \label{eqn:pendef}
  \end{equation}
  In addition, for any fixed function $f \in \F_i$, $\P(|\emprisk[n_i(T)](f) -
  \risk(f)| > \const_2 \epsilon) \le \const_1 \exp(-4 n_i(T) \epsilon^2)$.
\end{assumption}

\subsection{Some illustrative examples}
\label{sec:nestingexamples}
\newcommand{\rademacher}{\mathfrak{R}}

We now provide two concrete examples to illustrate
Assumptions~\ref{assumption:inclusion}--\ref{assumption:uniformity}.

\begin{example}[Linear classification with nested balls]
  \label{example:classification-balls}
  In a classification problem, each sample $\samp_i$ consists of a
  covariate vector $x \in \R^d$ and label $y \in \{-1,+1\}$. In
  margin-based linear classification, the predictions are the sign of
  the linear function $f_\theta(x) = \<\theta,x\>$, where $\theta \in
  \R^d$.  A natural sequence of model classes is sets $\{f_\theta\}$
  indexed via norm-balls of increasing radii: $\F_i = \{f_\theta :
  \theta \in \R^d, \ltwo{\theta} \le r_i\}$, where $0 \le r_1 < r_2 <
  \ldots$. By inspection, $\F_i \subset \F_{i+1}$ so that this
  sequence satisfies Assumption~\ref{assumption:inclusion}.
  
  The empirical and expected risks of a function $f_\theta$ are often
  measured using the sample average and expectation, respectively, of
  a convex upper bound on the 0-1 loss $\indic{yf_\theta(x) \le
    0}$. Examples of such losses include the hinge loss,
  $\ell(yf_\theta(x)) = \max(0,1-yf_\theta(x))$, or the logistic loss,
  \mbox{$\ell(yf_\theta(x)) = \log(1+\exp(-yf_\theta(x)))$}. Assume
  that $\E[\ltwo{x}^2] \le \xbound^2$ and let $\sigma_i$ be
  independent uniform $\{\pm 1\}$-valued random variables. Then we may
  use a penalty function $\pen_i$ based on Rademacher complexity
  $\rademacher_n(\F_i)$ of the class $i$,
  \begin{equation*}
    \rademacher_n(\F_i) \defeq \bigg\{ \frac{1}{n} \E\bigg[\sup_{f \in
        \F_i} \bigg|\sum_{i=1}^n \sigma_i f(X_i)\bigg| \bigg]\bigg\}
    \le \frac{2r_i \xbound}{\sqrt{n}}.
  \end{equation*}
  Setting $\pen_i$ to be the Rademacher complexity
  $\rademacher_n(\F_i)$ satisfies the conditions of
  Assumption~\ref{assumption:uniformity}~\cite{BartlettMe02} for both
  the logistic and the hinge losses which are 1-Lipschitz. Hence,
  using the standard Lipschitz contraction bound~\cite[Theorem
    12]{BartlettMe02}, we may take $\pen_i(T) =
  \frac{2r_i\xbound}{\sqrt{n_i(T)}}$.

  To illustrate Assumption~\ref{assumption:budget}, we take stochastic
  gradient descent~\cite{RobbinsMo51} as an example. Assuming that the
  computation time to process a sample $\samp$ is equal to the dimension $d$,
  then Nemirovski et al.~\cite{NemirovskiJuLaSh09} show that the computation
  time required by this algorithm to output a function $f = \falg{i}{T}$
  satisfying Assumption~\ref{assumption:budget}(d) (that is, a
  $\pen_i$-optimal empirical minimizer) is at most
  \begin{equation*}
    \frac{4r_i^2 \xbound^2}{\pen^2_i(T)} \cdot d.
  \end{equation*}
  Substituting the bound on $\pen_i(T)$
  above, we see that the computational time for class $i$ is at most
  $dn_i(T)$. In other words, given a computational time $T$, we can
  satisfy the Assumption~\ref{assumption:budget} by setting $n_i(T) \propto
  T/d$ for each class $i$---the number of samples remains constant
  across the hierarchy in this example.
\end{example}

\begin{example}[Linear classification in increasing dimensions]
  \label{example:classification-dim}
  Staying within the linear classification domain, we
  index the complexity of the model classes $\F_i$ by an increasing sequence of
  dimensions $\{d_i\} \subset \N$. Formally, we set
  \begin{equation*}
  \F_i = \{f_\theta : \theta_j = 0\mbox{~for~}j >
  d_i,~~\|\theta\|_2 \leq r_i\},
  \end{equation*}
  where $0 < r_1 \le r_2 \le \ldots$.  This structure captures a variable
  selection problem where we have a prior ordering on the covariates.

  In special scenarios, such as when the design matrix $X = [x_1 ~ x_2 ~
    \cdots ~ x_n]$ satisfies certain incoherence or irrepresentability
  assumptions~\cite{BuhlmannVa11}, variable selection can be performed using
  $\ell_1$-regularization or related methods. However, in general
  an oracle inequality for variable selection requires some form of
  exhaustive search over subsets. In the sequel, we show that in this simpler
  setting of variable selection over nested subsets, we can provide oracle
  inequalities without computing an estimator for each subset and without
  any assumptions on the design matrix $X$.

  For this function hierarchy, we consider complexity penalties arising
  from VC-dimension arguments~\cite{VapnikCh71,BartlettMe02}, in
  which case we may set
  \begin{equation*}
    \pen_i(T) = \sqrt{\frac{d_i}{n_i(T)}}
  \end{equation*}
  which satisfies Assumption~\ref{assumption:uniformity}.  Using
  arguments similar to those for
  Example~\ref{example:classification-balls}, we may conclude that the
  computational assumption~\ref{assumption:budget} can be satisfied
  for this hierarchy, where the algorithm $\alg$ requires time
  $d_in_i(T)$ to select $f \in \F_i$. Thus, given a computational
  budget $T$, we set the number of samples $n_i(T)$ for class $i$ to
  be proportional to $T/d_i$.
\end{example}

We provide only classification examples above since they demonstrate
the essential aspects of our formulation.  Similar quantities can also
be obtained for a variety of other problems, such as parametric and
non-parametric regression, and for a number of model hierarchies
including polynomial or Fourier expansions, wavelets, or Sobolev
classes, among others (for more instances, see,
e.g.~\cite{Massart03,BarronBiMa99,BartlettBoLu02}).

\subsection{The computationally-aware model selection algorithm}

Having specified our assumptions and given examples satisfying them,
we turn to describing our first computationally-aware model selection
algorithm. Let us begin with the simpler scenario where we have only
$K$ model classes (we extend this to infinite classes
below). Perhaps the most obvious computationally budgeted model
selection procedure is the following: allocate a budget of $T/K$ to
each model class $i$. As a result, class $i$'s estimator $\femp{i} =
\falg{i}{T/K}$ is computed using $n_i(T/K)$ samples.  Let $\wt{f}_n$
denote the output of the basic model selection
algorithm~\eqref{eqn:model-selection} with the choices $n = n_i(T/ K)$,
using $n_i(T/K)$ samples to evaluate the empirical risk for class
$i$, and modifying the penalty $\pen_i$ to be $\penbar_i(n) = \pen_i(n) +
\sqrt{\log i / n}$.
Then very slight modifications of standard arguments~\cite{Massart03,
  BartlettBoLu02} yield the oracle inequality
\begin{equation*}
  \risk(\wt{f}_n) \leq \min_{i=1,\dots,K}\left(\risk[i]^* +
  c\pen_i\left(\frac{T}{K}\right) + \sqrt{\frac{\log
      i}{n_i(T/K)}}\right)
\end{equation*}
with high probability, where $c$ is a universal constant. This approach
can be quite poor. For instance, in
Example~\ref{example:classification-dim}, we have $n_i(T/K) =
T/(Kd_i)$, and the above inequality incurs a penalty that grows as
$\sqrt{K}$. This is much worse than the logarithmic scaling in $K$
that is typically possible in computationally unconstrained
settings~\cite{BartlettBoLu02}. It is thus natural to ask whether we
can use the nested structure of our model hierarchy to allocate
computational budget more efficiently.

To answer this question, we introduce the notion of \emph{coarse-grid
  sets}, which use the growth structure of the complexity penalties
$\pen_i$, to construct a scheme for allocating the budget across the
hierarchy. Recall the constant $\const_2$ from
Assumption~\ref{assumption:uniformity} and let $m > 0$ be an arbitrary
constant (we will see that $m$ controls the probability of error in
our results). Given $\gridsizeplain \in \N$ ($\gridsizeplain \ge 1$),
we define
\begin{equation}
  \penbar_i(T, \gridsizeplain) \defeq
  2\pen_i\left(\frac{T}{\gridsizeplain}\right) + \const_2
  \sqrt{\frac{2(m + \log\gridsizeplain)}{n_i(T/\gridsizeplain)}}.
  \label{eqn:penbar}
\end{equation}
Notice that, to simplify the notation,
we hide the dependence of $\penbar_i$ on $m$.
With the definition~\eqref{eqn:penbar}, we now give a definition
characterizing the growth characteristics of the penalties and sample
sizes.
\begin{definition}
  \label{defn:gridsetdefn}
  Given a budget $T$, for a set $\gridset \subseteq \N$, we say that
  $\gridset$ satisfies the \emph{coarse grid condition} with
  parameters $\penscalefac$, $m$, and $\gridsizeplain$ if $|\gridset|
  = \gridsizeplain$ and for each $i$ there is an index $j \in
  \gridset$ such that
  \begin{equation}
     \penbar_i(T, \gridsizeplain) \le
     \penbar_j(T, \gridsizeplain) \le
     (1 + \penscalefac) \penbar_i(T, \gridsizeplain).
    \label{eqn:gridsetdefn}
  \end{equation}
\end{definition}
\noindent
Figure~\ref{fig:coarsegrid} gives an illustration of the coarse-grid set.
For simplicity in presentation, we set $\penscalefac = 1$ in the
statements of our results in the sequel.

\begin{figure}
  \centering
  \includegraphics[width=0.7\textwidth]{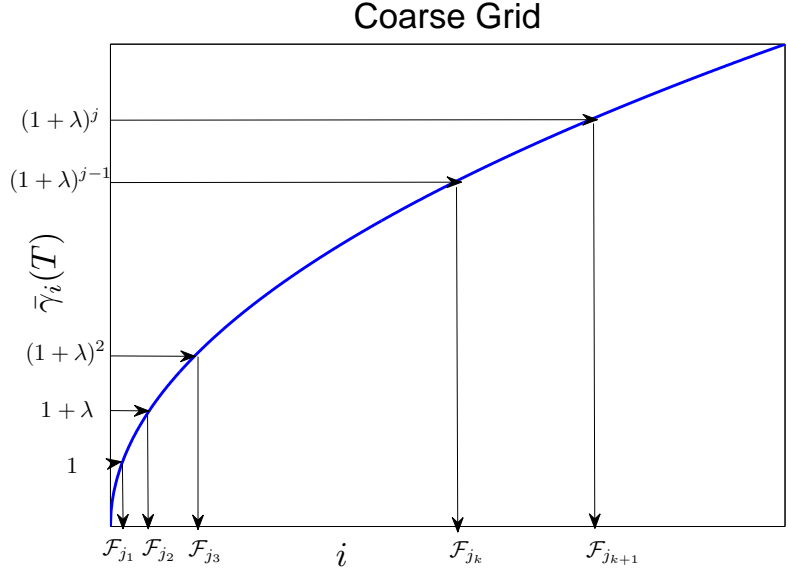}
  \caption{Construction of the coarse-grid set
    $\gridset_\penscalefac$. The $X$-axis is the class index $i$, and
    the $Y$-axis represents the corresponding complexity
    $\penbar_i(T)$. When the penalty function grows steeply early on,
    we include a large number of models. The number of complex models
    included in $\gridset_\penscalefac$ can be significantly smaller
    as the growth of penalty function tapers out.}
  \label{fig:coarsegrid}
\end{figure}

If the coarse-grid set is finite and,
say, $|\gridset| = \gridsizeplain$, then the set $\gridset$ presents a
natural collection of indices over which to perform model
selection. We simply split the budget uniformly amongst the
coarse-grid set $\gridset$, giving budget $T/\gridsizeplain$ to each
class in the set. Indeed, the main theorem of this section shows that
for a large class of problems, it always suffices to restrict our
attention to a finite grid set $\gridset$, allowing us to present both
a computationally tractable estimator and a good oracle inequality for
the estimator. In some cases, there may be no finite coarse grid
set. Thus we look for way to restrict our selection to finite sets,
which we can do with the following assumption (the assumption is
unnecessary if the hierarchy is finite).
\begin{assumption}
  \label{ass:boundedrisk}
  \begin{enumerate}[(a)]
  \item
  There is a constant $\bound < \infty$ such that $\risk[1]^*\le\bound$.
  \item For all $n \in \N$ the penalty function
  $\pen_1(n)\ge 1/n$.
  \end{enumerate}
\end{assumption}
\noindent
Assumption~\ref{ass:boundedrisk}(a) is satisfied, for example, if 
the loss function is bounded, or even if there is a function $f \in \F_1$
with finite risk.
Assumption~\ref{ass:boundedrisk}(b) also is mild;
unless the class $\F_1$ is trivial, in general classes satisfying
Assumption~\ref{assumption:uniformity}
have $\gamma_1(n)=\Omega(1/\sqrt n)$.

Under these assumptions,
we provide our computationally budgeted model selection
procedure in Algorithm~\ref{alg:coarsegrid}.
We will see in the proof of Theorem~\ref{theorem:nesting} below that
the assumptions ensure that we can build a coarse grid of size
\begin{equation*}
  \gridsizeplain = 
  \ceil{\log_2\left(1 + \bound n_1(T)\right)} + 2.
\end{equation*}
In particular, Assumption~\ref{assumption:budget}(d) ensures that
the complexity penalties continue to increase with the class index
$i$. Hence, there is a class $\maxgridplain$ such that the complexity
penalty $\pen_{\maxgridplain}$ is larger than the penalized risk
of the smallest class $\F_1$, at which point no class larger than
$\maxgridplain$ can be a minimizer in the oracle inequality. The
above choice of $\gridsizeplain$ ensures that there is at least one
class $j \in \gridset$ so that $j \geq \maxgridplain$, allowing us to
restrict our attention only to the function classes \mbox{$\{\F_i\mid
i\in \gridset\}$}.

\begin{algorithm}[t]
  \begin{algorithmic}
    \REQUIRE Model hierarchy $\{\F_i\}$ with corresponding penalty
    functions $\pen_i$, computational budget $T$, upper bound $\bound$ on
    the minimum risk of class 1, and confidence parameter $m$.

    \COMMENT{\emph{Construction of the coarse-grid set
      $\gridset$:}}
    \STATE Set $\gridsizeplain= \ceil{\log_2\left(1 + \bound n_1(T)\right)} + 2$.

    \FOR{$k = 0$ to $\gridsizeplain - 1$}
    \STATE Set $j_{k+1}$ to be the largest class for which
    $\penbar_j(T/\gridsizeplain) \leq 2^k\penbar_1(T/\gridsizeplain)$. 
    \ENDFOR

    \STATE Set $\gridset = \{j_k~:~k = 1,\ldots,\gridsizeplain\}$.

    \COMMENT{\emph{Model selection estimate:}}

    \STATE Set $\femp{i} = \falg{i}{T / \gridsizeplain}$ for $i \in
    \gridset$.

    \STATE Select a class
    $\what{i}$ that satisfies
    \begin{equation}
      \what{i} \in \argmin_{i \in
        \gridset}\,\left\{\,\emprisk[n_i(T/\gridsizeplain)](\femp{i}) +
      \pen_i\left(T/\gridsizeplain\right) +
      \frac{\const_2}{2}\sqrt{\frac{m}{n_i(T/\gridsizeplain)}}
      + \frac{\const_2}{2} \sqrt{\frac{\log \gridsizeplain}{n_i(T /
      \gridsizeplain)}}
      \,\right\}.
      \label{eqn:coarseestimate}
    \end{equation}

    \STATE Output the function
    $f = \femp{\iopt} = \falg{\iopt}{T/\gridsizeplain}$. 
  \end{algorithmic}
  \caption{Computationally budgeted model
    selection over nested hierarchies}
  \label{alg:coarsegrid}
\end{algorithm}

\subsection{Main result and some consequences}
\label{sec:nestingslowresults}

With the above definitions in place, we can now provide an oracle
inequality on the performance of the model selected by
Algorithm~\ref{alg:coarsegrid}. We start with our main theorem, and
then provide corollaries to help explain various aspects of it.

\begin{theorem}
  \label{theorem:nesting}
  Let $f =\alg(\what{i},\timesteps/\gridsizeplain)$ be the output of the
  algorithm $\alg$ for the class $\what{i}$ specified by the
  procedure~\eqref{eqn:coarseestimate}.  Let
  Assumptions~\ref{assumption:inclusion}--\ref{ass:boundedrisk} be
  satisfied. With probability at least $1 - 2\const_1\exp(-m)$
  \begin{equation}
    \risk(f) \leq \min_{i = 1,2,3,\ldots} \left\{ \risk[i]^* +
    4\pen_i\left(\frac{\timesteps}{\gridsizeplain}\right) +
    \const_2\sqrt{\frac{8(m +
        \log\gridsizeplain)}{n_i(T/\gridsizeplain)}}\right\} .
    \label{eqn:nestingbound}
  \end{equation}
  Furthermore, if $n_1(T) = \order(T)$ then $\gridsizeplain=\order(\log
  T)$.
\end{theorem}

The assumption that $n_1(T)$ is linear is mild:
unless $\F_1$ is trivial, any algorithm for $\F_1$ must at least
observe the data, and hence must use computation at least linear
in the sample size.

\paragraph{Remarks:}
To better understand the result of Theorem~\ref{theorem:nesting}, we turn to
a few brief remarks.
\begin{enumerate}[(a)]
\item We may ask what an omniscient oracle with
  access to the same computational algorithm $\alg$ could do. Such an
  oracle would know the optimal class $i^*$ and allocate the entire
  budget $T$ to compute $\falg{i^*}{T}$. By
  Assumption~\ref{assumption:uniformity}, the output $f$ of this
  oracle satisfies, with probability at least $1 - \const_1\exp(-m)$,
  \begin{equation}
    \risk(f) \leq \risk[i^*]^* + \pen_{i^*}(T) +
    \const_2\sqrt{\frac{m}{n_{i^*}(T)}}
    = \min_{i=1,2,3,\ldots}\left\{
    \risk[i]^* + \pen_{i}(T) + \const_2\sqrt{\frac{m}{n_{i}(T)}}
    \right\}. 
    \label{eqn:oraclebound}
  \end{equation}
  Comparing this to the right hand side of the inequality of
  Theorem~\ref{theorem:nesting}, we observe that not knowing the
  optimal class incurs a penalty in the computational budget of
  roughly a factor of $\gridsizeplain$.  This penalty is
  only logarithmic in the computational budget in most settings of
  interest.  
\item
  Algorithm~\ref{alg:coarsegrid} and Theorem~\ref{theorem:nesting}, as
  stated, require a priori knowledge of the computational budget $T$.
  We can address this using a standard doubling argument (see
  e.g.~\cite[Sec. 2.3]{CesaBianchiLu06}). Initially we assume $T = 1$
  and run Algorithm~\ref{alg:coarsegrid} accordingly. If we do not
  exhaust the budget, we assume $T = 2$, and rerun
  Algorithm~\ref{alg:coarsegrid} for another round. If there is more
  computational time at our disposal, we update our guess to $T = 4$
  and so on. Suppose the real budget is $T_0$ with \mbox{$2^k - 1 <
    T_0 \leq 2^{k+1}-1$}. After $i$ rounds of this doubling strategy,
  we have exhausted a budget of $2^{i-1}$, with the last round getting
  a budget of $2^{i-2}$ for $i \geq 2$. In particular, the last round
  with a net budget of $T_0$ is of length at least $T_0/4$. Since
  Theorem~\ref{theorem:nesting} applies to each individual round, we
  obtain an oracle inequality where we replace $T_0$ with $T_0/4$; we
  can be agnostic to the prior knowledge of the budget at the expense
  of slightly worse constants.
\item For ease of presentation, Algorithm~\ref{alg:coarsegrid} and
  Theorem~\ref{theorem:nesting} use a specific setting of the
  coarse-grid size, which corresponds to setting $\penscalefac = 1$ in
  Definition~\ref{defn:gridsetdefn}. In our proofs, we establish the
  theorem for arbitrary $\penscalefac > 0$. As a consequence, to obtain
  slightly sharper bounds, we may optimize this choice of $\penscalefac$; we
  do not pursue this here.
\end{enumerate}

Now let us turn to a specialization of Theorem~\ref{theorem:nesting}
to the settings outlined in
Examples~\ref{example:classification-balls}
and~\ref{example:classification-dim}. The following corollary shows
oracle inequalities under the computational restrictions that are only
logarithmically worse than those possible in the computationally
unconstrained model selection procedure~\eqref{eqn:model-selection}.
\begin{corollary}
  \label{corollary:classification}
  Let $m \ge 0$ be a specified constant.
  \begin{enumerate}[(a)]
  \item In the setting of Example~\ref{example:classification-balls},
    define $n$ so that
    $n T / d$ is the number of samples that can be processed
    by the inference algorithm $\alg$ using $T$ units of computation.
  Assume that $T$ is large enough that $nT\ge d/\bound$ and
  $nT\ge d/(4r_1^2\xbound^2)$.
  With probability at least $1 - 2 \const_1 \exp(-m)$, the output $f$ of
  Algorithm~\ref{alg:coarsegrid} satisfies
    \begin{equation*}
      \risk(f) \le \inf_{i = 1, 2, \ldots} \left\{\risk[i]^* +
      \sqrt{\frac{d \log_2(16\bound nT/d)}{nT}} \left(8 r_i \xbound
      + \sqrt{8} \const_2 \sqrt{m + \log\log_2(16\bound n T/d)}
      \right)\right\}.
    \end{equation*}
  \item In the setting of Example~\ref{example:classification-dim},
    define $n$ so that $nT / d_i$ is the number of samples that
    can be processed by the inference algorithm $\alg$ using
    $T$ units of computation.
  Assume that $T$ is large enough that $n T \ge 1$ and $n T \ge
  d_1/\bound$.  With probability at least $1 - 2 \const_1 \exp(-m)$,
  the output $f$ of Algorithm~\ref{alg:coarsegrid} satisfies
    \begin{equation*}
      \risk(f) \le \inf_{i = 1, 2, \ldots} \left\{\risk[i]^* +
      \sqrt{\frac{d_i \log_2(16\bound nT/d_1)}{nT}}
      \left( 4\sqrt{d_i} + \sqrt{8} \const_2
      \sqrt{m + \log\log_2(16\bound n T/d_1)}
      \right)\right\}.
    \end{equation*}
  \end{enumerate}
\end{corollary}


\subsection{Proofs}
\label{sec:nestingproofs}

As remarked after Theorem~\ref{theorem:nesting}, we will present our
proofs for general settings of $\lambda > 0$. For the proofs of
Theorem~\ref{theorem:nesting} and
Corollary~\ref{corollary:classification} in this slight
generalization, we define $\gridset_\penscalefac$ as a set satisfying
the coarse grid condition with parameters $\lambda$, $m$ and
$\gridsize$, with $\gridsize$ satisfying
\begin{equation}
  \gridsize \ge
  \ceil{\frac{\log\left(1 + \frac{\bound}{\penbar_1(T,
\gridsize)}\right)}{
      \log(1+\penscalefac)}}
  + 2.
  \label{eqn:gridsize}
\end{equation}

First, we show that this inequality is ensured
by the choice given in Algorithm~\ref{alg:coarsegrid}.
To see this, notice that
  \begin{align*}
    \ceil{\frac{\log\left(1 + \frac{\bound}{\penbar_1(T, \gridsize)}\right)}{
      \log(1+\penscalefac)}} + 2
    & \le \ceil{\frac{\log\left(1 + \frac{\bound}{\pen_1(T, 1)}\right)}{
      \log(1+\penscalefac)}} + 2 \\
    & = \ceil{\frac{\log\left(1 + \bound n_1(T)\right)}{
      \log(1+\penscalefac)}} + 2.
  \end{align*}
Thus, for $\penscalefac=1$, choosing
$\gridsize = \ceil{\log_2\left(1 + \bound n_1(T)\right) } + 2$ suffices.

We require the additional notation
\begin{equation}
  \maxgrid \defeq \max\{j~:~j \in \gridset_\penscalefac\},
  \label{eqn:maxgriddefn}
\end{equation}
where
\begin{equation}
  \label{eqn:Slambda}
  \gridset_\penscalefac=\{j_1,\ldots,j_{\gridsize}\}
\end{equation}
is the natural generalization of the set
$\gridset$ defined in Algorithm~\ref{alg:coarsegrid}: $j_{k+1}$ is
chosen as the largest index for which
$\penbar_j(T/\gridsize) \leq (1+\lambda)^k\penbar_1(T/\gridsize)$. 
We begin the proof of Theorem~\ref{theorem:nesting} by showing that
any $\gridsize$ satisfying~\eqref{eqn:gridsize} ensures that any class
$j > \maxgrid$ must have penalty too large to be optimal, so we can
focus on classes $j \le \maxgrid$.  We then show that the output $f$
of Algorithm~\ref{alg:coarsegrid} satisfies an oracle inequality for
each class in $\gridset_\penscalefac$, which is possible by an
adaptation of arguments in prior work~\cite{BartlettBoLu02}.  Using
the definition of our coarse grid set
(Definition~\ref{defn:gridsetdefn}), we can then infer an oracle
inequality that applies to each class $j \leq \maxgrid$, and our
earlier reduction to a finite model hierarchy completes the argument.

\subsubsection{Proof of Theorem~\ref{theorem:nesting}}
\label{sec:nestingtheorem}

First we show that the selection of the set $\gridset_\penscalefac$
satisfies Definition~\ref{defn:gridsetdefn}.
\begin{lemma}
  \label{lemma:selection-is-fine}
  Let $\{\pen_i\}$ be a sequence of increasing positive numbers and
  for each $k \in \{0, \ldots, \gridsizeplain-1\}$ set $j_{k+1}$ to be
  the largest index $j$ such that $\pen_j \le (1 + \penscalefac)^k
  \pen_1$. Then for each $i \in \N$ such that $i \le j_k$, there
  exists a $j \in \{j_1, \ldots, j_k\}$ such that $\pen_i \le \pen_j
  \le (1 + \penscalefac) \pen_i$.
\end{lemma}
\begin{proof}
  Let $i \le j_k$ and choose the smallest $j \in \{j_1, j_2, \ldots,
  j_k\}$ such that $\pen_i \le \pen_j$.  Assume for the sake of
  contradiction that $(1 + \penscalefac) \pen_i < \pen_j$. There
  exists some $k' \in \{0, \ldots, \gridsizeplain-1\}$ such that
  $\pen_j \le (1 + \penscalefac)^{k'} \pen_1$ and $\pen_j \geq (1 +
  \penscalefac)^{k'-1} \pen_1$, and thus we obtain
  \begin{equation}
    \pen_i < \frac{\pen_j}{1 + \penscalefac}
    \le (1 + \penscalefac)^{k'-1} \pen_1.
    \label{eqn:classtoosmall}
  \end{equation}
  Let $j'$ be the largest element smaller than $j$ in the collection
  $\{j_1, j_2, \dots, j_k\}$. Then by our construction, $j'$ is the
  largest index satisfying $\pen_{j'} \leq
  (1+\penscalefac)^{k'-1}\pen_1$. In particular, combining with our
  earlier inequality~\eqref{eqn:classtoosmall} leads to the conclusion
  that $i \leq j'$, which contradicts the fact that $j$ is the
  smallest index in $\{j_1, \ldots, j_k\}$ satisfying $\pen_i \le
  \pen_j$.
\end{proof}
\noindent
Next, we show that, for $\gridsize$
satisfying~\eqref{eqn:gridsize}, once the complexity penalty of a
class becomes too large, it can never be the minimizer of the penalized risk in
the oracle inequality~\eqref{eqn:nestingbound}. See
Appendix~\ref{app:nesting} for the proof.
\begin{lemma}
  Fix $\lambda>0$ and $m>0$, recall the definition~\eqref{eqn:maxgriddefn} of
  $\maxgrid$, and let $i^*$ be a class that
  attains the minimum in the right side of the
  bound~\eqref{eqn:nestingbound}.  We have $i^* \leq \maxgrid$.
  \label{lemma:finite}
\end{lemma}
\noindent
Equipped with the lemmas, we can restrict our attention only to
classes $i \in\gridset_\penscalefac$. To that end, the next result
establishes an oracle inequality for our algorithm compared to all
the classes in this set.

\begin{proposition}
  \label{prop:nesting-grid}
  Let $f = \femp{\iopt}$ be the function chosen from the class
  $\what{i}$ selected by the procedure~\eqref{eqn:coarseestimate},
  where $\gridset=\gridset_\penscalefac$ and
  $\gridsizeplain=\gridsize$.  Under the conditions of
  Theorem~\ref{theorem:nesting}, with probability at least $1 -
  2\const_1\exp(-m)$
  \begin{equation*}
    \risk(f) \leq \min_{i \in \gridset_\penscalefac} \left\{ \risk[i]^*
    + 2\pen_i\left(\frac{\timesteps}{\gridsize}\right) +
    \const_2\sqrt{\frac{2(m +
        \log\gridsize)}{n_i(T/\gridsize)}}\right\} .
  \end{equation*}
\end{proposition}
\noindent
The proof of the proposition follows from an argument similar to that
given in~\cite{BartlettBoLu02}, though we must carefully reason about
the different number of independent samples used to estimate within each class
$\F_i$. We present a proof in Appendix~\ref{app:nesting}. We can now
complete the proof of Theorem~\ref{theorem:nesting} using the proposition.

\begin{proof-of-theorem}[\ref{theorem:nesting}]
  Let $i$ be any class (not necessarily in $\gridset_\penscalefac$)
  and $j \in \gridset_\penscalefac$ be the smallest class satisfying
  $j \geq i$. Then, by construction of $\gridset_\penscalefac$, we
  know from Lemma~\ref{lemma:selection-is-fine} that
  \begin{align*}
    2\pen_i\left(\frac{T}{\gridsize}\right) + \const_2\sqrt{\frac{2(m +
        \log\gridsize)}{n_i(T/\gridsize)}}
    & \leq
    2\pen_j\left(\frac{T}{\gridsize}\right) + \const_2\sqrt{\frac{2(m +
        \log\gridsize)}{n_j(T/\gridsize)}} \\
    & \leq (1+\penscalefac)
    \left[2\pen_i\left(\frac{T}{\gridsize}\right) +
      \const_2\sqrt{\frac{2(m +
          \log\gridsize)}{n_i(T/\gridsize)}}\right].
  \end{align*}
  In particular, we can lower bound the penalized risk of class $i$ as 
  \begin{equation*}
    \risk[i]^* + (1+\penscalefac)
    \left[2\pen_i\left(\frac{T}{\gridsize}\right) +
      \const_2\sqrt{\frac{2(m +
          \log\gridsize)}{n_i(T/\gridsize)}}\right] \geq \risk[j]^* +
    2\pen_j\left(\frac{T}{\gridsize}\right) + \const_2\sqrt{\frac{2(m +
        \log\gridsize)}{n_j(T/\gridsize)}},
  \end{equation*}
  where we used the inclusion assumption~\ref{assumption:inclusion} to
  conclude that $\risk[j]^* \le \risk[i]^*$. Now
  applying Proposition~\ref{prop:nesting-grid}, the above lower bound,
  and Lemma~\ref{lemma:finite} in turn, we see that with probability at
  least $1 - 2\const_1\exp(-m)$
  \begin{align*}
    \risk(f) &\leq \min_{j \in \gridset_\penscalefac} \left\{ \risk[j]^*
    + 2\pen_j\left(\frac{\timesteps}{\gridsize}\right) +
    \const_2\sqrt{\frac{2(m +
        \log\gridsize)}{n_j(T/\gridsize)}}\right\} \\
    &\leq \min_{i = 1,2,\ldots,\maxgrid} \left\{ \risk[i]^*
    + (1+\penscalefac)\left(2\pen_i\left(\frac{\timesteps}{\gridsize}\right) +
    \const_2\sqrt{\frac{2(m +
        \log\gridsize)}{n_i(T/\gridsize)}}\right)\right\}\\
    &\leq \inf_{i = 1,2,3,\ldots} \left\{ \risk[i]^*
    + (1+\penscalefac)\left(2\pen_i\left(\frac{\timesteps}{\gridsize}\right) +
    \const_2\sqrt{\frac{2(m +
        \log\gridsize)}{n_i(T/\gridsize)}}\right)\right\}.
  \end{align*}
  For $\penscalefac=1$ (which we have seen
  satisfies~\eqref{eqn:gridsize}),
  this is the desired statement of the theorem.
\end{proof-of-theorem}

\subsubsection{Proof of Corollary~\ref{corollary:classification}} 

Under the conditions of Example~\ref{example:classification-balls},
and the assumption that $nT\ge d/(4 r_1^2\xbound^2)$,
Assumptions~\ref{assumption:inclusion}-\ref{ass:boundedrisk}
are satisfied with $n_i(T)=nT/d$ and
$\gamma_i(T)=2r_i\xbound\sqrt{d/(nT)}$.
(In particular, $nT\ge d/(4 r_1^2\xbound^2)$ implies that $\gamma_1$
satisfies Assumption~\ref{ass:boundedrisk}(b).)
Also, since $nT\ge d/\bound$, we have
\begin{equation*}
  s = \ceil{\log_2\left(1+\frac{\bound nT}{d}\right)}
  + 2
  \le \log_2(2\bound nT/d)+3
  = \log_2(16\bound nT/d).
\end{equation*}
Substituting into Theorem~\ref{theorem:nesting} gives the first part
of the corollary.

Similarly, under the conditions of
Example~\ref{example:classification-dim}
and the assumption that $nT\ge 1$,
Assumptions~\ref{assumption:inclusion}-\ref{ass:boundedrisk}
are satisfied with $n_i(T)=nT/d_i$ and
$\gamma_i(T)=d_i/\sqrt{nT}$.
(In particular, $nT\ge 1$ implies that $\gamma_1$
satisfies Assumption~\ref{ass:boundedrisk}(b).)
Also, since $nT\ge d_1/\bound$, we have
$s\le\log_2(16\bound nT/d)$ as before.
Substituting into Theorem~\ref{theorem:nesting} gives the second part
of the corollary.

\section{Fast rates for model selection}
\label{sec:nesting-fast}

Looking at the result given by Theorem~\ref{theorem:nesting}, we
observe that irrespective of the dependence of the penalties $\pen_i$
on the sample size, there are terms in the oracle inequality that
always decay as
$\order(1/\sqrt{n_i(T/\gridsize)})$.
A similar phenomenon is noted in~\cite{Bartlett08}
for classical model selection results in
computationally unconstrained settings; under
conditions similar to Assumption~\ref{assumption:uniformity}, this
inverse-root dependence on the number of samples is the best possible,
due to lower bounds on the fluctuations of the empirical process
(e.g.~\cite[Theorem 2.3]{BartlettMe06}). On the other hand, under
suitable low noise conditions~\cite{MammenTs99} or curvature
properties of the risk
functional~\cite{BartlettBoMe05,Koltchinskii06a,BartlettJoMc06}, it is
possible to obtain estimation guarantees of the form
\begin{equation*}
  \risk(\hat{f}) = \risk(f^*) + \order_p\left(\frac{1}{n}\right), 
\end{equation*}
where $\hat{f}$ (approximately) minimizes the $n$-sample empirical
risk. Under suitable assumptions, complexity regularization can
also achieve fast rates for model
selection~\cite{Bartlett08,Koltchinskii06}.
In this section, we show that similar results
can be obtained in computationally constrained inferential settings.

\subsection{Assumptions and example}

We begin by modifying our concentration assumption and
providing a motivating example.
\begin{assumption}
  For each $i$, let $f_i^* \in \argmin_{f \in \F_i} \risk(f)$. Then
  there are constants $\const_1, \const_2 > 0$ such that for any
  budget $T$ and the corresponding sample size $n_i(T)$
  \begin{subequations}
    \begin{align}
      \P\left[\sup_{f \in \F_i}\left(\risk(f) - \risk(\fopt{i}) -
        2(\emprisk[n_i(T)](f) - \emprisk[n_i(T)](\fopt{i}))\right) >
        \pen_i(T) + \const_2\epsilon\right] \leq
      \const_1\exp(-n_i(T)\epsilon). 
      \label{eqn:uniform-excess1}\\
      \P\left[\sup_{f \in \F_i}\left(\emprisk[n_i(T)](f) -
        \emprisk[n_i(T)](\fopt{i}) - 2(\risk(f) - \risk(\fopt{i}))\right) >
        \pen_i(T) + \const_2\epsilon\right] \leq
      \const_1\exp(-n_i(T)\epsilon).
      \label{eqn:uniform-excess2}
    \end{align}
  \end{subequations}
  \label{ass:uniform-excess}
\end{assumption}
\noindent
Contrasting this with our earlier
Assumption~\ref{assumption:uniformity}, we see that the probability
bounds~\eqref{eqn:uniform-excess1} and~\eqref{eqn:uniform-excess2}
decay exponentially in $\epsilon$ rather than $\epsilon^2$, which
leads to faster sub-exponential rates for estimation procedures.
Concentration inequalities of this form are now well
known~\cite{BartlettBoMe05,Koltchinskii06a,BartlettJoMc06}, and the
paper~\cite{Bartlett08} uses an identical assumption.

Before continuing, we give an example to illustrate the assumption.
\begin{example}[Fast rates for classification]
  We consider the function class hierarchy based on increasing
  dimensions of Example~\ref{example:classification-dim}. We assume
  that the risk $\risk(f_\theta) = \E[\ell(y, f_\theta(x))]$ and that
  the loss function $\ell$ is either the squared loss $\ell(y,
  f_\theta(x)) = (y - f_\theta(x))^2$ or the exponential loss from
  boosting $\ell(y, f_\theta(x)) = \exp(-yf_\theta(x))$. Each of
  these examples satisfies Assumption~\ref{ass:uniform-excess} with
  \begin{equation}
    \pen_i(T) = c\,\frac{d_i\log(n_i(T)/d_i)}{n_i(T)}, 
  \label{eqn:gammaifastexample}
  \end{equation}
  for a universal constant $c$.  This follows from Theorem 3
  of~\cite{Bartlett08} (which in turn follows from Theorem 3.3
  in~\cite{BartlettBoMe05} combined with an argument based on Dudley's
  entropy integral~\cite{Dudley99}).  The other parameter settings and
  computational considerations are identical to those of
  Example~\ref{example:classification-dim}.
  \label{example:classification-fast}
\end{example}

If we define $\femp{i} = \falg{i}{T}$, then using
Assumption~\ref{assumption:budget}(d) (that
$\emprisk[n_i(T)](\femp{i}) - \emprisk[n_i(T)](\fopt{i}) \le
\pen_i(T)$) in conjunction with
Assumption~\eqref{eqn:uniform-excess1}, we can conclude that for any
time budget $T$, with probability at least $1 - \const_1\exp(-m)$,
\begin{equation}
  \risk(\femp{i}) \leq \risk(\fopt{i}) + 3\pen_i(T) + \frac{\const_2
    m}{n_i(T)}. 
  \label{eqn:emp-to-exp-fast}
\end{equation}
One might thus expect that by following arguments similar to
those in~\cite{Bartlett08}, it would be possible to show fast rates
for model selection based on
Algorithm~\ref{alg:coarsegrid}. Unfortunately, the results
of~\cite{Bartlett08} heavily rely on the fact that the data used for
computing the estimators $\femp{i}$ is the same for each class $i$, so
that the fluctuations of the empirical processes corresponding to the
different classes are positively correlated. In our computationally
constrained setting, however, each class's estimator is computed on a
different sample.  It is thus more difficult to
relate the estimators than in previous work, necessitating a
modification of our earlier Algorithm~\ref{alg:coarsegrid} and a new
analysis, which follows.

\subsection{Algorithm and oracle inequality}
\label{sec:fastrates}

As in Section~\ref{sec:nesting}, our approach is based on performing
model selection over a coarsened version of the collection $\F_1,
\F_2, \ldots$.  To construct the coarser collection of indices, we
define the composite penalty term (based on
Assumption~\ref{ass:uniform-excess})
\begin{equation}
  \penbar_i(T, \gridsizeplain) \defeq
  \pcona\pen_i\left(\frac{T}{\gridsizeplain}\right) + \pconb\frac{\const_2
    m + 2\log\gridsizeplain}{n_i(T/\gridsizeplain)}.
  \label{eqn:penbar-fast}
\end{equation}
Based on the above penalty term, we define our analogue of the coarse grid
set~\eqref{eqn:gridsetdefn}.

We give our modified model selection procedure in
Algorithm~\ref{alg:coarsegrid-fast}. In the algorithm and in our
subsequent analysis, we use the shorthand $\emprisk[i](f)$ to denote
the empirical risk of the function $f$ on the $n_i(T)$ samples
associated with class $i$.  Our main oracle inequality is the
following:

\begin{algorithm}[t]
  \begin{algorithmic}
    \REQUIRE Model hierarchy $\{\F_i\}$ with corresponding penalty
    functions $\pen_i$, computational budget $T$, upper bound $\bound$ on
    the minimum risk of class 1, and confidence parameter $m>0$.

    \COMMENT{\emph{Construction of the coarse-grid set
      $\gridset$:}}
    \STATE Set $\gridsizeplain= \ceil{\log_2\left(1 + \bound n_1(T)\right)} + 2$.

    \FOR{$k = 0$ to $\gridsizeplain - 1$}
    \STATE Set $j_{k+1}$ to be the largest class for which
    $\penbar_j(T/\gridsizeplain) \leq 2^k\penbar_1(T/\gridsizeplain)$. 
    \ENDFOR

    \STATE Set $\gridset = \{j_k~:~k = 1,\ldots,\gridsizeplain\}$.

    \COMMENT{\emph{Model selection estimate:}}

    \STATE Set $\femp{i} = \falg{i}{T / \gridsizeplain}$ for $i \in
    \gridset$.

    \STATE Select the class $\what{i} \in \gridset_\penscalefac$ to be
    the largest class that satisfies 
    \begin{equation}
      \emprisk[\iopt](\femp{\iopt}) + \pconc \pen_{\iopt}
      \left(\frac{T}{\gridsizeplain}\right) + \pcond \const_2
      \left(\frac{m + \log \gridsizeplain}{n_{\iopt}(T /
      \gridsizeplain)}\right)
      \le
      \emprisk[\iopt](\femp{j}) + 
      \pconc\pen_j\left(\frac{T}{\gridsizeplain}\right)
      \label{eqn:coarseestimate-fast}
    \end{equation}
    for all $j \in \gridset$ such that $j < \what{i}$. 

    \STATE Output the function
    $\falg{\what{i}}{T/\gridsizeplain}$. 
  \end{algorithmic}
  \caption{Computationally
    budgeted model selection over hierarchies with fast concentration}
  \label{alg:coarsegrid-fast}
\end{algorithm}

\begin{theorem}
  \label{theorem:nesting-fast}
  Let $f =\alg(\what{i},\timesteps/\gridsize)$ be the output of the
  algorithm $\alg$ for class $\what{i}$ specified by the
  procedure~\eqref{eqn:coarseestimate-fast}.  Let
  Assumptions~\ref{assumption:inclusion},~\ref{assumption:budget},~\ref{ass:boundedrisk}
  and~\ref{ass:uniform-excess} be satisfied. With probability at least
  $1 - 2\const_1\exp(-m)$
  \begin{equation}
    \risk(f) \leq \inf_{i = 1,2,3,\ldots} \left\{ \risk[i]^* +
    40\gridsizeplain\pen_i\left(\frac{\timesteps}{\gridsizeplain}\right)
    + 10 \gridsizeplain \const_2\frac{m +
      \log\gridsizeplain}{n_i(T/\gridsizeplain)}\right\} .
    \label{eqn:nestingbound-fast}
  \end{equation}
  Furthermore, if $n_1(T) = \order(T)$ then $\gridsizeplain = \order(\log T)$.
\end{theorem}
By inspection of the bound~\eqref{eqn:emp-to-exp-fast}---achieved by
devoting the full computational budget $T$ to the optimal class---we see
that Theorem~\ref{theorem:nesting-fast}'s oracle inequality has
dependence on the computational budget within logarithmic factors of the
best possible.

The following corollary shows the application of
Theorem~\ref{theorem:nesting-fast} to the classification problem we
discuss in Example~\ref{example:classification-fast}.
\begin{corollary}
  \label{corollary:classification-fast}
  In the setting of Example~\ref{example:classification-fast},
    define $n$ so that $nT / d_i$ is the number of samples that
    can be processed by the inference algorithm $\alg$ using
    $T$ units of computation.
  Assume that $nT\ge ed_1^2$, $nT\ge d_1/\bound$, and choose the
  constant $c$ in the definition~\eqref{eqn:gammaifastexample}
  of $\pen_i(T)$ such that $c\ge 1/d_1$.
  With probability at least $1 - 4 \const_1 \exp(-m)$,
  the output $f$ of Algorithm~\ref{alg:coarsegrid-fast}
  satisfies
  \begin{align*}
    \risk(f) & \leq \inf_{i = 1, 2, \ldots} \left\{ \risk[i]^* + 
    \frac{10d_i\log_2^2(16\bound nT/d_1)}{nT}
    \left(4cd_i\log\left(\frac{nT}{d_1^2\log_2(16\bound nT/d_1)}\right)
    \right.\right. \\*
    & \qquad\qquad\qquad\qquad\qquad\qquad\qquad\qquad
    \left.\left. \rule{0mm}{7mm} {}
    +\const_2\left(m+\log\log_2(16\bound nT/d_1)\right)
    \right)\right\}.
  \end{align*}
\end{corollary}

\subsection{Proofs of main results}
\label{sec:proofs-nesting-fast}

In this section, we provide proofs of
Theorem~\ref{theorem:nesting-fast} and
Corollary~\ref{corollary:classification-fast}. Like our previous proof
for Theorem~\ref{theorem:nesting}, we again provide the proof of
Theorem~\ref{theorem:nesting-fast} for general settings of $\lambda >
0$. The proof of Theorem~\ref{theorem:nesting-fast} broadly follows
that of Theorem~\ref{theorem:nesting}, in that we establish an
analogue of Proposition~\ref{prop:nesting-grid}, which provides an
oracle inequality for each class in the coarse-grid set
$\gridset_\penscalefac$.  We then extend the proven inequality to
apply to each function class $\F_i$ in the hierarchy using the
definition~\eqref{eqn:gridsetdefn} of the grid set.


\begin{proof-of-theorem}[\ref{theorem:nesting-fast}]
  Let $n_i$ be shorthand for $n_i(T / \gridsize)$, the number of samples
  available to class $i$, and let $\emprisk[i](f)$ denote the empirical risk
  of the function $f$ using the $n_i$ samples for class $i$. In addition, let
  $\pen_i(n_i)$ be shorthand for $\pen_i(n_i(T / \gridsize))$, the penalty
  value for class $i$ using $n_i(T / \gridsize)$ samples.  With these
  definitions, we adopt the following shorthand for the events in the
  probability bounds~\eqref{eqn:uniform-excess1}
  and~\eqref{eqn:uniform-excess2}.  Let $\epsilon = \{\epsilon_i\}$ be an
  $\gridsize$-dimensional vector with (arbitrary for now) positive entries.
  For each pair of indices $i$ and $j$ define
  \begin{subequations}
    \begin{align}
      \event{1}{ij}{\epsilon_i} & \defeq
      \bigg\{\sup_{f \in \F_j} \left(\risk(f) - \risk(\fopt{j})
      - 2\left(
      \emprisk[i](f) - \emprisk[i](\fopt{j})
      \right)\right)
      \le \pen_j(n_i) + \const_2 \epsilon_i
      \bigg\}
      \label{eqn:excess-event-1} \\
      \event{2}{ij}{\epsilon_i} & \defeq
      \bigg\{\sup_{f \in \F_j} \left(
      \emprisk[i](f) - \emprisk[i](\fopt{j})
      - 2\left(\risk(f) - \risk(\fopt{j})\right)\right)
      \le \pen_j(n_i) + \const_2 \epsilon_i
      \bigg\},
      \label{eqn:excess-event-2}
    \end{align}
  \end{subequations}
  and define the joint events
  \begin{equation}
    \unionevent{1}{\epsilon}
    \defeq \bigcup_{i \in \gridset_\penscalefac}
    \bigcup_{j \in \gridset_\penscalefac}
    \event{1}{ij}{\epsilon_i}
    ~~~ \mbox{and} ~~~
    \unionevent{2}{\epsilon}
    \defeq \bigcup_{i \in \gridset_\penscalefac}
    \bigcup_{j \in \gridset_\penscalefac}
    \event{2}{ij}{\epsilon_i}.
    \label{eqn:uniform-excess-event}
  \end{equation}

  With the ``good'' events~\eqref{eqn:uniform-excess-event} defined, we turn
  to the two technical lemmas, which relate the risk of the chosen function
  $\femp{\iopt}$ to $\fopt{i}$ for each $i \in \gridset_\penscalefac$.
  We provide proofs of
  both lemmas in Appendix~\ref{app:nesting-fast}. To make the proofs of each
  of the lemmas cleaner and see the appropriate choices of constants, we
  replace the selection strategy~\eqref{eqn:coarseestimate-fast} with one
  whose constants have not been specified. Specifically, we select $\iopt$ as
  the largest class that satisfies
  \begin{equation}
    \emprisk[\iopt](\femp{\iopt}) + \pconcunspec \pen_{\iopt}
    \left(\frac{T}{\gridsize}\right) + \pcondunspec \const_2 \epsilon_{\iopt}
    \le
    \emprisk[\iopt](\femp{j}) + 
    \pconcunspec\pen_j\left(\frac{T}{\gridsize}\right)
    \label{eqn:coarseestimate-fast-unspecified}
  \end{equation}
  for $j \in \gridset$ with $j \le \iopt$.
  \begin{lemma}
    \label{lemma:nesting-fast-lower}
    Let the events~\eqref{eqn:excess-event-1}
    and~\eqref{eqn:excess-event-2} hold for all $i, j \in
    \gridset_\penscalefac$, that is, $\unionevent{1}{\epsilon}$ and
    $\unionevent{2}{\epsilon}$ hold. Then using the selection
    strategy~\eqref{eqn:coarseestimate-fast-unspecified}, for each $j
    \le \iopt$ with $j \in \gridset_\penscalefac$ we have
    \begin{equation*}
      \risk(\femp{\iopt}) \le
      \risk(\fopt{j})
      + \half\left[
        \left(\frac{17}{2} - \pconcunspec\right) \pen_{\iopt}(n_{\iopt})
        + (6 + \pconcunspec) \pen_j(n_j)
        + 2 \const_2 \epsilon_j
        + \left(\frac{9}{2} - \pcondunspec\right)
        \const_2 \epsilon_{\iopt}\right].
    \end{equation*}
  \end{lemma}
  \noindent
  We require a different argument for the case that $j \ge \iopt$,
  and the constants are somewhat worse.
  \begin{lemma}
    \label{lemma:nesting-fast-upper}
    Let the events~\eqref{eqn:excess-event-1}
    and~\eqref{eqn:excess-event-2} hold for all $i, j \in
    \gridset_\penscalefac$, that is, $\unionevent{1}{\epsilon}$ and
    $\unionevent{2}{\epsilon}$ hold. Assume also that $\pconcunspec
    \ge 17/2$ and $\pcondunspec \ge 7/2$.  Then using the selection
    strategy~\eqref{eqn:coarseestimate-fast-unspecified}, for each $j
    \ge \iopt$ with $j \in \gridset_\penscalefac$ we have
    \begin{equation*}
      \risk(\femp{\iopt}) \leq \risk(\fopt{j}) +
      \gridsize\left[(2 \pconcunspec + 3)\pen_j(n_j) 
      + (2 \pcondunspec + 1) \epsilon_j\right].
    \end{equation*}
  \end{lemma}

  We use Lemmas~\ref{lemma:nesting-fast-lower}
  and~\ref{lemma:nesting-fast-upper} to complete the proof of the
  theorem.  When Assumption~\ref{ass:uniform-excess} holds, the probability
  that one of the events $\unionevent{1}{\epsilon}$ and
  $\unionevent{2}{\epsilon}$ fails to hold is upper bounded
  by
  \begin{equation*}
    \P(\unionevent{1}{\epsilon}^c \cup \unionevent{2}{\epsilon}^c)
    \le \sum_{i,j \in \gridset_\penscalefac}
    \P(\event{1}{ij}{\epsilon_i}^c)
    + \sum_{i,j \in \gridset_\penscalefac}
    \P(\event{2}{ij}{\epsilon_i}^c)
    \le 2 \const_1 \sum_{i,j \in \gridset_\penscalefac}
    \exp(-n_i(T / \gridsize) \epsilon_i)
  \end{equation*}
  by a union bound.
  Thus, we see that if we define the constants
  \begin{equation*}
    \epsilon_i = 2 \cdot \frac{m + \log(\gridsize)}{n_i(T / \gridsize)},
  \end{equation*}
  we obtain that all of the events $\event{1}{ij}{\epsilon_i}$ and
  $\event{2}{ij}{\epsilon_i}$ hold with probability at least $1 - 2 \const_1
  \exp(-m)$. Applying Lemmas~\ref{lemma:nesting-fast-lower}
  and~\ref{lemma:nesting-fast-upper} with the
  choices $\pconcunspec = \pconc$ and $\pcondunspec = \pcond$,
  we obtain that with probability at least $1 - 2 \const_1 \exp(-m)$
  \begin{equation}
    \label{eqn:nesting-fast-grid}
    \risk(\femp{\iopt})
    \le \min_{i \in \gridset_\penscalefac}
    \left\{\risk(\fopt{i}) + \gridsize \left(20 \pen_i(n_i)
    + 10 \frac{m + \log(\gridsize)}{n_i(T / \gridsize)}\right)
    \right\}.
  \end{equation}

  The inequality~\eqref{eqn:nesting-fast-grid} is the analogue of
  Proposition~\ref{prop:nesting-grid} in the current setting.  Given the
  inequality, the remainder of the proof of Theorem~\ref{theorem:nesting-fast}
  follows the same recipe as that of Theorem~\ref{theorem:nesting}. Recalling
  the notation~\eqref{eqn:maxgriddefn} defining $\maxgrid$, we apply the
  inequality~\eqref{eqn:nesting-fast-grid} with the
  definition of the grid set~\eqref{eqn:Slambda} to obtain an oracle
  inequality compared to all classes $i \leq \maxgrid$. Then provided
  that
\begin{equation*}
  \gridsize \ge
  \ceil{\frac{\log\left(1 + \frac{\bound}{
        \gridsize\penbar_1(T, \gridsize)}\right)}{
      \log(1+\penscalefac)}}
  + 2,
\end{equation*}
we can transfer the result to the entire model hierarchy as before.
For $\lambda=1$, the choice of $\gridsizeplain$ employed in
Algorithm~\ref{alg:coarsegrid-fast} again suffices for this.
\end{proof-of-theorem}

\begin{proof-of-corollary}[\ref{corollary:classification-fast}]
  In the setting of Example~\ref{example:classification-fast},
  we set $n_i(T)=nT/d_i$ and
  \begin{equation*}
    \pen_i(T) =
    \frac{cd_i\log(n_i(T)/d_i)}{n_i(T)}
    = \frac{cd_i^2\log(nT/d_i^2)}{nT}.
  \end{equation*}
  It is straightforward to verify that the conditions of the corollary
  ensure that Assumptions~\ref{assumption:inclusion},~\ref{assumption:budget},~\ref{ass:boundedrisk}
  and~\ref{ass:uniform-excess} are satisfied. In particular, $nT\ge
  ed_1^2$ and $c\ge 1/d_1$ ensure that $\pen_1(T)\ge 1/n_1(T)$.
  Also, $nT\ge d_1/\bound$ ensures that
  $\gridsizeplain \le \log_2(16\bound nT/d_1)$.
  Substituting $\pen_i$, $n_i$ and $\gridsizeplain$
  into Theorem~\ref{theorem:nesting-fast} gives the result.
\end{proof-of-corollary}

\section{Oracle inequalities for unstructured models}
\label{sec:bandits}

To this point, our results have addressed the model selection problem
in scenarios where we have a nested collection of models. In the most
general case, however, the collection of models may be quite
heterogeneous, with no relationship between the different model
families. In classification, for instance, we may consider generalized
linear models with different link functions, decision trees, random
forests, or other families among our collection of models. For a
non-parametric regression problem, we may want to select across a
collection of dictionaries such as wavelets, splines, and
polynomials. While this more general setting is obviously more
challenging than the structured cases in the prequel, we would like to
study the effects that limiting computation has on model selection
problems, understanding when it is possible to outperform
computation-agnostic strategies.


\subsection{Problem setting and algorithm}
\label{sec:bandits-setup}

When no structure relates the models under consideration, it is
impossible to work with an infinite collection of classes within a
finite computational time---any estimator must evaluate each class
(that is, at least one sample must be allocated to each class, as any
class could be significantly better than the others). As a result, we
restrict ourselves to finite model collections in this section, so
that we have a sequence $\F_1,\dots,\F_K$ of models from which we wish
to select.
Our approach to the unstructured case is to incrementally allocate
computational quota amongst the function classes, where we trade off
receiving samples for classes that have good risk performance against
exploring classes for which we have received few data points.  More
formally, with $\timesteps$ available quanta of computation, it is
natural to view the model selection problem as a $\timesteps$ round
game, where in each round a procedure selects a function class $i$ and
allocates it one additional quantum of computation.

With this setup, we turn to stating a few natural assumptions.  We
assume that the computational complexity of fitting a model grows
linearly and incrementally with the number of samples, which means
that allocating an additional quantum of training time allows the
learning algorithm $\alg$ to process an additional $n_i$ samples for
class $\F_i$. In the context of Sections~\ref{sec:nesting}
and~\ref{sec:nesting-fast}, this means that we assume $n_i(t) = tn_i$
for some fixed number $n_i$ specific to class $i$. This linear growth
assumption is satisfied, for instance, when the loss function $\loss$
is convex and the black-box learning algorithm $\alg$ is a stochastic
or online convex optimization
procedure~\cite{CesaBianchiLu06,NemirovskiJuLaSh09}.  We also require
assumptions similar to
Assumptions~\ref{assumption:budget} and~\ref{assumption:uniformity}:
\begin{assumption}
  \label{assumption:uniform-bandits}
  Let $\falg{i}{T} \in \F_i$ denote the output of algorithm $\alg$ when
  executed for class $\F_i$ with a computational budget $T$.
  \begin{enumerate}[(a)]
  \item
    \label{assumption:uniformity-time-bandits}
    For each $i$, there exists an $n_i \in \N$ such that in $T$ units
    of time, algorithm $\alg$ can compute $\falg{i}{T}$ using $n_i T$ samples.
  \item
    \label{assumption:uniformity-alg-bandits}
    For each $i \in [K]$, there is a function $\pen_i$ and constants
    $\const_1, \const_2 > 0$ such that for any $T \in \N$,
    \begin{equation}
      \P\left(|\emprisk[n_iT](\falg{i}{T}) - \risk(\falg{i}{T})| >
      \pen_i(n_iT) + \const_2 \eps\right) \leq \const_1 \exp(-4 n_iT
      \eps^2).
      \label{eqn:pendef-bandits}
    \end{equation}
  \item
    \label{assumption:uniformity-min-bandits}
    The output $\falg{i}{T}$ is
    a $\pen_i(n_iT)$-minimizer of $\emprisk[n_iT]$, that is, 
    \begin{equation*}
      \emprisk[n_iT](\falg{i}{n_iT}) - \inf_{f \in \F_i}
      \emprisk[n_iT](f) \le \pen_i(n_iT).
    \end{equation*}
  \item \label{assumption:uniformity-pen-bandits}
    For each $i$, the function $\pen_i$ satisfies
    $\pen_i(n) \le
    \penconstant_i n^{-\alpha_i}$ for some $\alpha_i > 0$.
  \item \label{assumption:uniformity-fixed-bandits}
    For any fixed function $f \in \F_i$,
    $\P(|\emprisk[n](f) - \risk(f)| > \const_2 \epsilon)
    \le \const_1 \exp(-4 n \epsilon^2)$.
  \end{enumerate}
\end{assumption}
\noindent
Comparing to Assumptions~\ref{assumption:budget}
and~\ref{assumption:uniformity}, we see that the main difference is in the
linear time assumption~\eqref{assumption:uniformity-time-bandits}
and growth assumption~\eqref{assumption:uniformity-pen-bandits}.
In addition,
the complexity penalties and function classes discussed in our
earlier examples satisfy Assumption~\ref{assumption:uniform-bandits}.

We now present our algorithm for successively
allocating computational quanta to the function classes.
To choose the class $i$ receiving computation at iteration $t$, the
procedure must balance competing goals of \emph{exploration}, evaluating
each function class $\F_i$ adequately, and \emph{exploitation}, giving more
computation to classes with low empirical risk.
To promote exploration, we use an optimistic selection criterion
to choose class $i$, which---assuming that $\F_i$ has
seen $n$ samples at this point---is
\begin{equation}
  \label{eqn:obj-criterion}
  \obj(i, n) = \emprisk[n](\falg{i}{n}) - \pen_i(n) - \sqrt{\frac{\log K}{n}}
  + \pen_i(\timesteps n_i).
\end{equation}
The intuition behind the definition of $\obj(i, n)$ is that we would
like the algorithm to choose functions $f$ and classes $i$ that
minimize $\emprisk[n](f) + \pen_i (\timesteps n_i) \approx \risk(f) +
\pen_i(\timesteps n_i)$, but the negative $\pen_i(n)$ and $\sqrt{\log
  K / n}$ terms lower the criterion significantly when $n$ is small
and thus encourage initial exploration. The
criterion~\eqref{eqn:obj-criterion} essentially combines a penalized
model-selection objective with an
optimistic criterion similar to those used in multi-armed bandit
algorithms~\cite{AuerCBF02}. Algorithm~\ref{alg:bandit-select}
contains the formal description of our bandit procedure for model
selection. Algorithm~\ref{alg:bandit-select} begins by
receiving $n_i$ samples for each of the $K$ classes $\F_i$ to form
the preliminary empirical estimates~\eqref{eqn:obj-criterion}; we then
use the optimistic selection criterion until
the computational budget is exhausted.

\begin{algorithm}[t]
  \begin{algorithmic}
    \STATE For each $i \in [K]$, query $n_i$ examples
    from class $\F_i$.
    \FOR{$t= K+1$ to $\timesteps$}
    \STATE Let $n_i(t)$ be the number of examples seen for class $i$ until 
    time $t$
    \STATE Let $\isubt = \argmin_{i\in [K]} \obj(j, n_i(t)) -
    \sqrt{\frac{\log t}{n_i(t)}}$.
    \STATE Query $n_{\isubt}$ examples for class $\isubt$.
    \ENDFOR
    \STATE Output $\what{i}$, the index of the most frequently
    queried class.
  \end{algorithmic}
  \caption{Multi-armed bandit algorithm for selection of best class
    $\what{i}$.} 
  \label{alg:bandit-select}
\end{algorithm}

\subsection{Main results and some consequences}

The goal of the selection procedure is to find the best penalized class $\opt$:
a class satisfying
\begin{equation*}
  \opt \in \argmin_{i \in [K]} \left\{
  \inf_{f \in \F_i} \risk(f) + \pen_i(\timesteps n_i)\right\}
  = \argmin_{i \in [K]} \left\{\risk[i]^* + \pen_i(\timesteps n_i)\right\}.
\end{equation*}
To present our main results for
Algorithm~\ref{alg:bandit-select}, we define the excess penalized risk
$\excess_i$ of class $i$:
\begin{equation}
  \label{eqn:excess}
  \excess_i \defeq \risk[i]^* + \pen_i(\timesteps n_i) - \risk[\opt]^*
  - \pen_{\opt}(\timesteps n_{\opt}) \ge 0.
\end{equation}
Without loss of generality, we assume that the infimum in
$\risk[i]^* = \inf_{f \in \F_i} \risk(f)$ is attained by
a function $f_i^*$ (if not, we use a limiting argument, choosing
some fixed $f_i^*$ such that $\risk(f_i^*) \le \inf_{f \in \F_i}
\risk(f) + \delta$ for an arbitrarily small $\delta > 0$).

The gains of a computationally adaptive strategy over na\"ive strategies are
clearest when the gap~\eqref{eqn:excess} is non-zero for each $i$, though in
the sequel, we forgo this requirement. Under this assumption, we can follow the
ideas of~\citet{AuerCBF02} to show that the fraction of the computational
budget allocated to any suboptimal class $i \neq \opt$ goes quickly to zero as
$\timesteps$ grows. We provide the proof of the following theorem in
Section~\ref{sec:ucbproof}.

\begin{theorem}
  \label{theorem:expected-pulls}
  Let Alg.~\ref{alg:bandit-select} be run for $\timesteps$ rounds, and
  let $T_i(t)$
  be the number of times class $i$ is queried through round $t$.  Let
  $\excess_i$ be defined as in \eqref{eqn:excess} and
  Assumption~\ref{assumption:uniform-bandits} hold, and assume that
  $\timesteps \ge K$. Define $\beta_i = \max\{1/\alpha_i, 2\}$.  There is a
  constant $C$ such that
  \begin{equation*}
    \E[T_i(\timesteps)] \le \frac{C}{n_i} \left(\frac{\penconstant_i +
      \const_2 \sqrt{\log \timesteps}}{\excess_i}\right)^{\beta_i}
    ~~~ \mbox{and} ~~~ 
    \P\left(T_i(\timesteps) > \frac{C}{n_i}
    \left(\frac{\penconstant_i + \const_2 \sqrt{\log
        \timesteps}}{\excess_i}\right)^{\beta_i} \right) \le
    \frac{\const_1}{T K^4},
  \end{equation*}
  where $\penconstant_i$ and $\alpha_i$ are the constants in the
  definition~\ref{assumption:uniform-bandits}(\ref{assumption:uniformity-pen-bandits})
  of the concentration function $\pen_i$.
\end{theorem}

At a high level, this result shows that the fraction of budget
allocated to any suboptimal class goes to 0 at the rate $\frac{1}{n_i
  \timesteps}\left(\frac{\sqrt{\log
    T}}{\excess_i}\right)^{\beta_i}$. Hence, asymptotically in
$\timesteps$, the procedure performs almost as if all the computational
budget were allocated to class $\opt$.
To see an example of concrete rates that can be concluded from the
above result, let $\F_1,\dots,\F_K$ be model classes with finite
VC-dimension,\footnote{Similar corollaries hold for any model class
  whose metric entropy grows polynomially in $\log\frac{1}{\epsilon}$.} so that
Assumption~\ref{assumption:uniform-bandits} is satisfied with $\alpha_i =
\frac{1}{2}$. Then we have
\begin{corollary}
  \label{corr:ucb-vc}
  Under the conditions of Theorem~\ref{theorem:expected-pulls}, assume
  $\F_1,\dots,\F_K$ are model classes of finite VC-dimension, where
  $\F_i$ has dimension $d_i$. Then there is a constant $C$ such that
  \begin{equation*}
    \E[T_i(\timesteps)] \le
    C\frac{\max\{d_i, \const_2^2 \log T\}}{\excess^2_in_i}
    ~~~ \mbox{and} ~~~
    \P\left(T_i(\timesteps) > C\frac{\max\{d_i,\const_2^2 \log
      T\}}{\excess^2_in_i} \right) \le \frac{\const_1}{T K^4}. 
  \end{equation*}
\end{corollary}

A lower bound by~\citet{LaiRo85} for the multi-armed bandit problem
shows that Corollary~\ref{corr:ucb-vc} is nearly optimal in
general. To see the connection, let $\F_i$ correspond to the $i$th arm
in a multi-armed bandit problem and the risk $\risk[i]^*$ be the
expected reward of arm $i$ and assume w.l.o.g.\ that $\risk[i]^* \in
[0, 1]$. In this case, the complexity penalty $\pen_i$ for each class
is 0. Let $p_i$ be a distribution on $\{0, 1\}$, where $p_i(1) =
\risk[i]^*$ and $p_i(0) = 1 - \risk[i]^*$ (let $p_i = p_i(1)$ for
shorthand). \citeauthor{LaiRo85} give a lower bound that shows that
the expected number of pulls of any suboptimal arm is at least $
\E[T_i(\timesteps)] = \Omega\left(\log T / \kl{p_i}{p_{\opt}}\right)$,
where $p_i$ and $p_{\opt}$ are the reward distributions for the $i$th
and optimal arms, respectively. An asymptotic expansion shows that
$\kl{p_i}{p_{\opt}} = \excess_i^2 / (2p_i(1 - p_i))$, plus higher
order terms, in this case; Corollary~\ref{corr:ucb-vc} is
essentially tight.

The condition that the gap $\excess_i > 0$ may not always be
satisfied, or $\excess_i$ may be so small as to render the bound in
Theorem~\ref{theorem:expected-pulls} vacuous.  Nevertheless, it is
intuitive that our algorithm can quickly find a small set of
``good'' classes---those with small penalized risk---and spend its
computational budget to try to distinguish amongst them. In this case,
Algorithm~\ref{alg:bandit-select} does not visit suboptimal
classes and so can output a function $f$ satisfying good oracle
bounds.  In order to prove a result quantifying this intuition, we
first upper bound the \emph{regret} of
Algorithm~\ref{alg:bandit-select}, that is, the average excess risk
suffered by the algorithm over all iterations, and then show how to
use this bound for obtaining a model with a small risk. For
the remainder of the section, we simplify the presentation by assuming
that $\alpha_i \equiv \alpha$ and define $\beta
= \max\{1/\alpha, 2\}$.
\begin{proposition}
  \label{proposition:regret-bound}
  Use the same assumptions as Theorem~\ref{theorem:expected-pulls},
  but further assume that $\alpha_i \equiv \alpha$ for all $i$.  With
  probability at least $1 - \const_1 / TK^3$, the regret (average
  excess risk) of Algorithm~\ref{alg:bandit-select} satisfies
  \begin{equation*}
    \sum_{i=1}^K \excess_i T_i(T) \le 2 e T^{1 - 1/\beta}
    \left(C\sum_{i=1}^K \frac{(\penconstant_i + \const_2\sqrt{\log
        T})^\beta}{ n_i} \right)^{1/\beta}
  \end{equation*}
  for a constant $C$ dependent on $\alpha$.
\end{proposition}

Our final main result builds on
Proposition~\ref{proposition:regret-bound} to show that when it is
possible to average functions across classes $\F_i$, we can aggregate
all the ``played'' functions $f_t$, one for each iteration $t$, to
obtain a function with small risk. Indeed, setting $f_t =
\falg{i_t}{n_{i_t}(t)}$, we obtain the following theorem (whose proof,
along with that of Proposition~\ref{proposition:regret-bound}, we
provide in Appendix~\ref{sec:no-separation}):
\begin{theorem}
  \label{theorem:no-separation}
  Use the conditions of Proposition~\ref{proposition:regret-bound}.
  Let the risk function $\risk$ be convex on $\F_1 \cup \ldots \cup
  \F_K$, and let $f_t$ be the function chosen by algorithm $\alg$ at
  round $t$ of Alg.~\ref{alg:bandit-select}. Define the average
  function $\favg = \frac{1}{T} \sum_{t=1}^T f_t$.  There are
  constants $C$, $C'$ (dependent on $\alpha$) such that with
  probability greater than $1 - 2 \const_2 / (TK^3)$,
  \begin{align*}
    \risk(\favg) & \le \risk^* + \pen_{i^*}(T n_{i^*})
    + 2 e \const_2 T^{-\beta}\sqrt{\log T}
    \left(\sum_{i=1}^K \frac{C}{n_i}\right)^{1/\beta} \\
    & \quad ~ + C'\, T^{-1/\beta}
    \left( \sum_{i=1}^K
    \left[c_i n_i^{-\alpha}
      + \const_2 n_i^{-\half} \sqrt{\log K} +
      \const_2 n_i^{-\half} \sqrt{\log T}\right]^{\beta}
    \right)^{1/\beta}.
  \end{align*}
\end{theorem}

Let us interpret the above bound and discuss its optimality.  When
$\alpha = \half$ (e.g., for VC classes), we have $\beta = 2$;
moreover, it is clear that $\sum_{i = 1}^K \frac{C}{n_i} =
\order(K)$. Thus, to within constant factors,
\begin{equation*}
  \risk(\favg) = \risk[\opt]^* + \pen_{\opt}(T n_{\opt}) +
  \order\left(\frac{\sqrt{K \max\{\log T, \log K\}}}{\sqrt{T}}\right).
\end{equation*}
Ignoring logarithmic factors, the above bound is minimax optimal,
which follows by a reduction of our model selection problem to the
special case of a multi-armed bandit problem. In this case, Theorem
5.1 of~\citet{auer03multiarmed} shows that for any set of
$K,\timesteps$ values, there is a distribution over the rewards of
arms which forces $\Omega(\sqrt{K\timesteps})$ regret, that is, the
average excess risk of the classes chosen by
Alg.~\ref{alg:bandit-select} must be $\Omega(\sqrt{K \timesteps})$,
matching Proposition~\ref{proposition:regret-bound} and
Theorem~\ref{theorem:no-separation}.

The scaling $\order(\sqrt{K})$ is essentially as bad as splitting
the computational budget $\timesteps$ uniformly across each of the
$K$ classes, which yields (roughly) an oracle inequality of the form
\begin{equation*}
  \risk(f) = \risk[\opt]^* + \pen_\opt(\timesteps n_\opt / K)
  + \order\bigg(\frac{\sqrt{K \log K}}{\sqrt{\timesteps n_\opt}}\bigg).
\end{equation*}
Comparing this bound to Theorem~\ref{theorem:no-separation}, we see
that the penalty $\pen_i$ in the theorem is smaller. The other key
distinction between the two bounds (ignoring logarithmic factors) is
the difference between
\begin{equation*}
  \sum_{i = 1}^K \frac{1}{n_i} ~~ \mbox{and} ~~
  \frac{K}{n_\opt}.
\end{equation*}
When the left quantity is smaller than the right, the bandit-based
Algorithm~\ref{alg:bandit-select} and the extension indicated by
Theorem~\ref{theorem:no-separation} give improvements over the na\"ive
strategy of uniformly splitting the budget across classes. However, if
each class has similar computational cost $n_i$, no strategy can
outperform the na\"ive one.

We also observe that we can apply the online procedure of
Algorithm~\ref{alg:bandit-select} to the nested setup of
Sections~\ref{sec:nesting} and~\ref{sec:nesting-fast} as well. In this
case, by applying Algorithm~\ref{alg:bandit-select} only to elements
of the coarse-grid set $\gridset_\penscalefac$, we can replace $K$ in
the bounds of Theorems~\ref{theorem:expected-pulls}
and~\ref{theorem:no-separation} with $\gridsize$, which gives results
similar to our earlier Theorems~\ref{theorem:nesting}
and~\ref{theorem:nesting-fast}. In particular, if we are in the setup
of Theorem~\ref{theorem:expected-pulls} with a large separation
between penalized risks, then Algorithm~\ref{alg:bandit-select}
applied to the coarse-grid set is expected to outperform a uniform
allocation of budget within the set as in Sections~\ref{sec:nesting}
and~\ref{sec:nesting-fast}.

\subsection{Proof of Theorem~\ref{theorem:expected-pulls}}
\label{sec:ucbproof}

At a high level, the proof of this theorem involves combining the
techniques for analysis of multi-armed bandits developed
by~\citet{AuerCBF02} with Assumption~\ref{assumption:uniform-bandits}. We
start by giving a lemma that will be useful to prove the theorem. The
lemma states that after a sufficient number of initial iterations
$\tau$, the probability that Algorithm~\ref{alg:bandit-select} chooses to
receive samples for a sub-optimal function class $i \neq \opt$ is
extremely small.  Recall also our notational convention that $\beta_i
= \max\{1/\alpha_i, 2\}$.
\begin{lemma}
  Let Assumption~\ref{assumption:uniform-bandits} hold.
  For any class $i$, any $s_i \in [1,T]$ and $s_{\opt} \in [\tau,T]$
  where $\tau$ satisfies 
  \begin{equation*}
    \tau > \frac{2^{\beta_i} (\penconstant_i + \const_2
      \sqrt{\log T} + \const_2 \sqrt{\log K})^{\beta_i}}{
      n_i\excess_i^{\beta_i}},
  \end{equation*}
  we have
  \begin{equation*}
    \P\left( \obj(i, n_i s_i) - \const_2 \sqrt{\frac{\log T}{n_is_i}}
      \le \obj(\opt, n_\opt s_\opt) - \const_2 \sqrt{\frac{\log
          T}{n_{\opt}s_{\opt}}}\right) \leq
      \frac{2\const_1}{(\timesteps K)^4}.
    \end{equation*}
  \label{lemma:ucbbadevent}
\end{lemma}

We defer the proof of the lemma to
Appendix~\ref{appendix:proof-ucb-bad-event}, though at a high level
the proof works as follows. The ``bad event'' in
Lemma~\ref{lemma:ucbbadevent}, which corresponds to
Algorithm~\ref{alg:bandit-select} selecting a sub-optimal class $i \neq
\opt$, occurs only if one of the following three errors occurs: the
empirical risk of class $i$ is much lower than its true risk, the
empirical risk of class $\opt$ is higher than its true risk, or $s_i$
is not large enough to actually separate the true penalized risks from
one another. The assumptions of the lemma
make each of these three sub-events
quite unlikely.  Now we turn to the proof of
Theorem~\ref{theorem:expected-pulls}, assuming the lemma.

Let $i_t$ denote the model class index $i$ chosen by
Algorithm~\ref{alg:bandit-select} at time $t$, and let
$\numpulls_i(t)$ denote the number of times class $i$ has been
selected at round $t$ of the algorithm. When no time index is needed,
$\numpulls_i$ will denote the same thing.  Note that if $i_t = i$ and
the number of times class $i$ is queried exceeds $\tau > 0$, then by
the definition of the selection criterion~\eqref{eqn:obj-criterion}
and choice of $i_t$ in Alg.~\ref{alg:bandit-select}, for some
$\numpulls_i \in \{\tau, \ldots, t - 1\}$ and $\numpulls_{\opt} \in
\{1, \ldots, t-1\}$ we have
\begin{equation*}
  \obj(i, n_i s_i) - \const_2 \sqrt{\frac{\log T}{n_i s_i}} \le
  \obj(\opt, n_\opt s_\opt) - \const_2 \sqrt{\frac{\log T}{n_\opt
      s_\opt}}.  
\end{equation*}
Here we interpret $\obj(i, n_i s_i)$ to mean a random realization of
the observed risk consistent with the samples we observe. Using the
above implication, we thus have
\begin{align}
  T_i(\timesteps) &= 1 + \sum_{t=K+1}^\timesteps \ind{i_t = i} \nonumber 
  ~~ \le ~~ \tau + \sum_{t = K+1}^\timesteps \ind{i_t = i, T_i(t-1) \geq
    \tau} \nonumber \\ & \le \tau + \sum_{t = K+1}^\timesteps \ind{
    \min_{\tau \le s_i < t} \obj(i, n_i s_i) - \const_2
    \sqrt{\frac{\log T}{n_i s_i}} \le \max_{0 < s < t} \obj(\opt,
    n_\opt s_\opt) - \const_2 \sqrt{\frac{\log T}{n_{\opt}s_{\opt}}}}
  \nonumber \\ & \le \tau + \sum_{t=1}^T 
  \sum_{s_\opt=1}^{t-1} \sum_{s_i=\tau}^{t-1} \ind{ \obj(i, n_i s_i) -
    \const_2 \sqrt{\frac{\log T}{n_is_i}} \le \obj(\opt, n_\opt
    s_\opt) - \const_2 \sqrt{\frac{\log T}{n_{\opt}s_{\opt}}}}.
  \label{eqn:select-infinite-bound}
\end{align}

To control the last term, we invoke Lemma~\ref{lemma:ucbbadevent} and
obtain that
\begin{align*}
\tau > \frac{2^{\beta_i} (\penconstant_i + \const_2 \sqrt{\log T} +
  \const_2 \sqrt{\log K})^{\beta_i}}{
  n_i\excess_i^{\beta_i}}~~\Rightarrow~~ \E[T_i(\timesteps)] \leq \tau
+ \sum_{t=1}^T \sum_{s=1}^{t-1}\sum_{s_i=\tau}^{t-1}
2\frac{\const_1}{(T K)^4}
  \le \tau + \frac{\const_1}{T K^4}.
\end{align*}
Hence for any suboptimal class $i \neq \opt$, $\E[T_i(n)] \leq \tau_i +
\const_1 / (T K^4)$, where $\tau_i$ satisfies the lower bound of
Lemma~\ref{lemma:ucbbadevent} and is thus logarithmic in $\timesteps$. Under
the assumption that $T \ge K$, for $i \neq \opt$,
\begin{equation}
  \E[T_i(\timesteps)]
  \le C \frac{(\penconstant_i + \const_2
    \sqrt{\log T})^{\max\{1/\alpha_i, 2\}}}{n_i
    \excess_i^{\max\{1/\alpha_i, 2\}}}
  \label{eqn:suboptimal-play-bound}
\end{equation}
for a constant $C \le 2 \cdot 4^{\max\{1/\alpha_i, 2\}}$. Now we prove
the high-probability bound. For this part, we need only concern
ourselves with the sum of indicators from
\eqref{eqn:select-infinite-bound}. Markov's inequality shows that
\begin{equation*}
  \P\left(\sum_{t = K + 1}^T \ind{i_t = i, T_i(t - 1) \ge \tau} \ge
  1 \right) \le \frac{\const_1}{T K^4}.
\end{equation*}
Thus we can assert that the bound~\eqref{eqn:suboptimal-play-bound} on
$T_i(\timesteps)$ holds with high probability.

\paragraph{Remark:}
By examining the proof of Theorem~\ref{theorem:expected-pulls}, it is
straightforward to see that if we modify the multipliers on the
square root terms in the criterion \eqref{eqn:obj-criterion} by
$m\const_2$ instead of $\const_2$, we get that the probability bound
is of the order $T^{3 - 4m^2} K^{-4m^2}$, while the bound on
$T_i(\timesteps)$ is scaled by $m^{1/\alpha_i}$.

\section{Discussion}
\label{sec:discussion}

In this paper, we have presented a new framework for model selection
with computational constraints. The novelty of our setting is the idea
of using computation---rather than samples---as the quantity against
which we measure the performance of our estimators.  \comment{ By
  carefully capturing the relative computational needs of fitting
  different models to our data, we are able to formalize the very
  natural intuition: \emph{Given a computational budget, a simple
    model can be fitted to a lot more samples than a complex model}.
} As our main contribution, we have presented algorithms for model
selection in several scenarios, and the common thread in each is that
we attain good performance by evaluating only a small and
intelligently-selected set of models, allocating samples to each model
based on computational cost. For model selection over nested
hierarchies, this takes the form of a new estimator based on a coarse
gridding of the model space, which is competitive (up to logarithmic
factors) with an omniscient oracle.  A minor extension of our
algorithm is adaptive to problem complexity, since it yields fast
rates for model selection when the underlying estimation problems have
appropriate curvature or low-noise properties.  We also presented an
exploration-exploitation algorithm for model selection in unstructured
cases, showing that it obtains (in some sense) nearly optimal
performance.

There are certainly many possible extensions and open questions that
our work raises. We address the setting where the complexity penalties
are known and can be computed easily in closed form. Often it is
desirable to use data-dependent
penalties~\cite{LugosiWe04,BartlettBoMe05,Massart03}, since they adapt
to the particular problem instance and data distribution.  It appears
to be somewhat difficult to extend such penalties to the procedures we
have developed in this paper, but we believe it would be quite
interesting. Another
natural question to ask is whether there exist intermediate model
selection problems between a nested sequence of classes and a
completely unstructured collection. Identifying other structures---and
obtaining the corresponding oracle inequalities and understanding
their dependence on computation---would be an interesting extension of
the results presented here.

More broadly, we believe the idea of using computation, in addition to the
number of samples available for a statistical inference problem, to measure
the performance of statistical procedures is appealling for a much broader
class of problems. In large data settings, one would hope that more data would
always improve the risk performance of statistical procedures, even with a
fixed computational budget.  We hope that extending these ideas to other
problems, and understanding how computation interacts with and affects the
quality of statistical estimation more generally, will be quite fruitful.



\section*{Acknowledgements}

We gratefully acknowledge illuminating discussions with Cl\'ement
Levrard, who helped us with earlier versions of this work and whose
close reading helped us clarify (and correct problems with) many of
our arguments.  In performing this research, Alekh Agarwal was
supported by a Microsoft Research Fellowship and Google PhD
Fellowship, and John Duchi was supported by the National Defense
Science and Engineering Graduate Fellowship (NDSEG) Program. Alekh
Agarwal and Peter Bartlett gratefully acknowledge the support of the
NSF under award DMS-0830410 and of the ARC under award FL1110281.

\appendix

\section{Auxiliary results for 
  Theorem~\ref{theorem:nesting} and
  Corollary~\ref{corollary:classification}} 
\label{app:nesting}

We start by establishing Lemma~\ref{lemma:finite}. To
prove the lemma, we first need a simple claim.
\begin{lemma}
  Let $c_1 > c_2 > 0$, $\gridsizeplain > 0$, and define
  \begin{align*}
    i_1^* &= \argmin_{i = 1,2,3,\ldots} \left\{ \risk[i]^* +
    c_1\left(\pen_i\left(\frac{\timesteps}{\gridsizeplain}\right)
    + \const_2\sqrt{\frac{2(m +
        \log\gridsizeplain)}{n_i(T/\gridsizeplain)}}\right)\right\},\\
    i_2^* &= \argmin_{i = 1,2,3,\ldots} \left\{ \risk[i]^* +
    c_2\left(\pen_i\left(\frac{\timesteps}{\gridsizeplain}\right) +
    \const_2\sqrt{\frac{2(m +
        \log\gridsizeplain)}{n_i(T/\gridsizeplain)}}\right)\right\}.
  \end{align*}
  Then under the monotonicity assumptions~\ref{assumption:budget}, we
  have $i_1^* \leq i_2^*$.
  \label{lemma:multiplemonotone}
\end{lemma}

\begin{proof}
  Recall the shorthand definition~\eqref{eqn:penbar} of
  $\penbar_i$. Under the
  monotonicity assumptions~\ref{assumption:budget}(a)--(b), $\penbar_i$ is
  monotone increasing in $i$. By the definitions of $i_1^*$ and
  $i_2^*$ we have
  \begin{align*}
    \risk[i_1]^* +
    c_1\penbar_{i_1^*}\left(T, \gridsizeplain\right) \leq
    \risk[i_2]^* +
    c_1\penbar_{i_2^*}\left(T, \gridsizeplain\right)
    ~~~ \mbox{and} ~~~
    \risk[i_2]^* +
    c_2\penbar_{i_2^*}\left(T, \gridsizeplain\right) \leq
    \risk[i_1]^* +
    c_2\penbar_{i_1^*}\left(T, \gridsizeplain\right).
  \end{align*}
  Adding the two inequalities we obtain
  \begin{equation*}
    (c_1 - c_2)\penbar_{i_1^*}\left(T, \gridsizeplain\right)
    \leq (c_1 -
    c_2)\penbar_{i_2^*}\left(T, \gridsizeplain\right). 
  \end{equation*}
  Since $c_1 - c_2 > 0$ by assumption, the monotonicity of
  $\penbar_i$ guarantees $i_1^* \leq i_2^*$. 
\end{proof}

\begin{proof-of-lemma}[\ref{lemma:finite}]
  Lemma~\ref{lemma:multiplemonotone} allows us to establish a simpler
  version of Lemma~\ref{lemma:finite}. Since $1 + \penscalefac > 1$,
  it suffices to establish $i_0 \leq \maxgrid$, where
  \begin{equation*}
    i_0 = \argmin_{i = 1,2,3,\ldots} \left\{ \risk[i]^* +
    \pen_i\left(\frac{\timesteps}{\gridsize}\right) +
    \const_2\sqrt{\frac{2(m +
        \log\gridsize)}{n_i(T/\gridsize)}}\right\} .
  \end{equation*}
  Let $\penbar_i$ be shorthand for the quantity~\eqref{eqn:penbar} as
  usual. Recalling the construction of $\gridset_\penscalefac$
  in~\eqref{eqn:Slambda}, we observe that any class $i > \maxgrid$
  satisfies
  \begin{equation*}
    (1 +
    \penscalefac)^{\gridsize-2}\,\penbar_1\left(T, \gridsize\right)
    < \penbar_i \left(T, \gridsize\right)
  \end{equation*}
  The setting~\eqref{eqn:gridsize} of $\gridsize$ ensures that 
  \begin{equation*}
    (1 + \penscalefac)^{\gridsize - 2}
    \ge (1 + \penscalefac)^{\ceil{\log(1 + \bound / \penbar_1(T, \gridsize))
          / \log(1 + \penscalefac)}}
    \ge \exp\left(\log\left(1 + \frac{\bound}{\penbar_1(T, \gridsize)}
    \right)\right)
  \end{equation*}
  so that
  \begin{equation*}
    (1 +
    \penscalefac)^{\gridsize-2}\,\penbar_1\left(T, \gridsize\right)
    \geq \bound + \penbar_1\left(T, \gridsize\right)
    \geq \risk[1]^* + \penbar_1\left(T, \gridsize\right)
    \geq \inf_{i} \left\{ \risk[i]^* +
    \penbar_i\left(\timesteps, \gridsize\right)\right\} . 
  \end{equation*}
  Hence we observe that for $i > \maxgrid$,
  \begin{align*}
    \risk[i]^* +
    \penbar_{i}\left(T, \gridsize\right) &\geq
    \penbar_{i}\left(T, \gridsize\right)\\ 
    & > (1 + \penscalefac)^{\gridsize -
      2}\,\penbar_1\left(T, \gridsize\right) \\
    & \geq \inf_{j \in \{1, 2, \ldots\}} \left\{ \risk[j]^* +
    \penbar_j\left(\timesteps, \gridsize\right)\right\} .  
  \end{align*}
  We must thus have $i_0 \leq \maxgrid$, and
  Lemma~\ref{lemma:multiplemonotone} further implies that $i^* \leq
  \maxgrid$. 
\end{proof-of-lemma}

We finally provide a proof for Proposition~\ref{prop:nesting-grid}. 

\begin{proof-of-proposition}[\ref{prop:nesting-grid}]
  Since for any $a, b \ge 0$, $\sqrt{a} + \sqrt{b} \le \sqrt{2(a + b)}$, it
  suffices to control the probability of the event
  \begin{equation}
    \risk(f) > \min_{i \in \gridset} \bigg\{
    \risk[i]^* + 2 \pen_i\left(\frac{\timesteps}{\gridsize}\right)
    + \const_2 \sqrt{\frac{\log \gridsize}{n_i(\timesteps / \gridsize)}}
    + \const_2 \sqrt{\frac{m}{n_i(\timesteps / \gridsize)}}\bigg\}.
    \label{eqn:nesting-event}
  \end{equation}
  For the event~\eqref{eqn:nesting-event} to occur, at least one of
  \begin{subequations}
    \begin{equation}
      \risk(f) > \min_{i \in
        \gridset} \left\{ \emprisk[n_i(T/\gridsize)](\femp{i}) +
      \pen_i\left(\frac{\timesteps}{\gridsize}\right) +
      \frac{\const_2}{2}\sqrt{\frac{m}{n_i(T/\gridsize)}}
      + \frac{\const_2}{2}
      \sqrt{\frac{\log \gridsize}{n_i(\timesteps / \gridsize)}}
      \right\}
      \label{eqn:nesting-union-bound-1}
    \end{equation}
    or
    \begin{align}
      \lefteqn{\min_{i \in \gridset} \left\{ \emprisk[
          n_i(T/\gridsize)](\femp{i}) +
        \pen_i\left(\frac{\timesteps}{\gridsize} \right) +
        \frac{\const_2}{2}
        \sqrt{\frac{\log\gridsize}{n_i(\timesteps/\gridsize)}} +
        \frac{\const_2}{2}\sqrt{\frac{m}{n_i(T/\gridsize)}}\right \}}
      \nonumber \\
      & \qquad ~ >
      \min_{i \in \gridset} \left\{ \risk[i]^* +
      2\pen_i\left(\frac{\timesteps}{\gridsize}\right) +
      \const_2\sqrt{\frac{\log\gridsize}{n_i(\timesteps/\gridsize)}} +
      \const_2\sqrt{\frac{m}{n_i(T/\gridsize)}}
      \right\}
      \label{eqn:nesting-union-bound-2}
    \end{align}
    must occur.
  \end{subequations}
  We bound the probabilities of the events~\eqref{eqn:nesting-union-bound-1}
  and~\eqref{eqn:nesting-union-bound-2} in turn.

  If the event~\eqref{eqn:nesting-union-bound-1} occurs, by definition
  of the selection strategy~\eqref{eqn:coarseestimate}, it must be the
  case that for some $i \in \gridset$ (namely $i = \what{i}$)
  \begin{equation*}
    \risk(\femp{i}) > \emprisk[n_i(\timesteps / \gridsize)](\femp{i})
    + \pen_i\left(\frac{\timesteps}{\gridsize}\right)
    + \frac{\const_2}{2} \sqrt{\frac{m}{n_i(\timesteps / \gridsize)}}
    + \frac{\const_2}{2}
    \sqrt{\frac{\log \gridsize}{n_i(\timesteps / \gridsize)}}
  \end{equation*}
  since the chosen $f$ minimizes the right side of this display over
  the classes $\F_i$ for $i \in \gridset$. By a union bound, we see
  that
  \begin{align*}
    \lefteqn{\P\left[ \risk(f) > \min_{i \in \gridset} \left\{
        \emprisk[n_i(T/\gridsize)](\femp{i}) +
        \pen_i\left(\frac{\timesteps}{\gridsize}\right) +
        \frac{\const_2}{2}\sqrt{\frac{m}{n_i(T/\gridsize)}}
        + \frac{\const_2}{2}
        \sqrt{\frac{\log \gridsize}{n_i(\timesteps / \gridsize)}}
        \right\}\right]} \\
    & \le \P\left[
      \exists ~ i \in \gridset ~ \mbox{s.t.}~
      \risk(\femp{i}) > \emprisk(\femp{i}) +
      \pen_i\left(\frac{T}{\gridsize}\right) +
      \frac{\const_2}{2}\sqrt{\frac{m}{n_i(T/\gridsize)}}
      + \frac{\const_2}{2}
      \sqrt{\frac{\log \gridsize}{n_i(\timesteps / \gridsize)}}
      \right] \\
    & \leq \const_1 \sum_{i \in \gridset} \exp\left(-m - \log \gridsize\right)
    = \const_1 \exp(-m),
  \end{align*}
  where the final inequality follows from
  Assumption~\ref{assumption:uniformity}.

  Now we bound the probability of the event~\eqref{eqn:nesting-union-bound-2},
  noting that the event implies that
  \begin{equation*}
    \max_{i \in S}\left\{ \emprisk[n_i(\timesteps / \gridsize)](\femp{i}) -
    \risk[i]^* - \pen_i\left(\frac{\timesteps}{\gridsize}\right) -
    \frac{\const_2}{2}\sqrt{\frac{\log\gridsize}{n_i(\timesteps/\gridsize)}} -
        \frac{\const_2}{2}\sqrt{\frac{m}{n_i(T/\gridsize)}}\right\}
    > 0.
  \end{equation*}
  We can thus apply a union bound to see that the probability
  of the event~\eqref{eqn:nesting-union-bound-2} is bounded by
  \begin{align}
    \lefteqn{\P\left[ \max_{i \in \gridset} \left\{
        \emprisk[n_i(\timesteps / \gridsize)](\femp{i}) -
        \risk[i]^* - \pen_i\left(\frac{\timesteps}{\gridsize}\right) -
        \frac{\const_2}{2}
        \sqrt{\frac{\log\gridsize}{n_i(\timesteps/\gridsize)}} -
        \frac{\const_2}{2}\sqrt{\frac{m}{n_i(T/\gridsize)}}\right\} 
        > 0\right]} \nonumber \\
    & \le \sum_{i \in \gridset}\P\left[\emprisk[n_i(\timesteps /
        \gridsize)](\femp{i}) - \risk[i]^* -
      \pen_i\left(\frac{\timesteps}{\gridsize}\right) -
      \frac{\const_2}{2}\sqrt{\frac{\log\gridsize}{n_i(\timesteps/\gridsize)}}
      - \frac{\const_2}{2}\sqrt{\frac{m}{n_i(T/\gridsize)}} >
      0\right] \nonumber \\
    & \leq \sum_{i \in \gridset}\P\left[ \emprisk(f_i^*)
      - \risk[i]^* > \frac{\const_2}{2} \sqrt{\frac{\log\gridsize}{
          n_i(\timesteps/\gridsize)}} +
          \frac{\const_2}{2}\sqrt{\frac{m}{n_i(T/\gridsize)}}
      \right], 
    \label{eqn:nesting-union-sum}
  \end{align}
  where the final inequality uses Assumption~\ref{assumption:budget}(d),
  which states that $\alg$ outputs a $\pen_i$-minimizer of the
  empirical risk. Now we can bound the deviations using
  the second part of Assumption~\ref{assumption:uniformity},
  since $f_i^*$ is non-random: the quantity~\eqref{eqn:nesting-union-sum}
  is bounded by
  \begin{align*}
    \sum_{i \in \gridset}
    \const_1\exp\left(-n_i(\timesteps/\gridsize)
      \left(\frac{\log\gridsize}{n_i(T/\gridsize)} +
      \frac{m}{n_i(T/\gridsize)}\right)\right) \leq \const_1\exp(-m).  
  \end{align*}
  Combining the two events~\eqref{eqn:nesting-union-bound-1}
  and~\eqref{eqn:nesting-union-bound-2} completes the proof of the
  proposition.
\end{proof-of-proposition}




\section{Auxilliary results for Theorem~\ref{theorem:nesting-fast}}
\label{app:nesting-fast}

\begin{proof-of-lemma}[\ref{lemma:nesting-fast-lower}]
  In the proof of the lemma, assume that both of the
  events~\eqref{eqn:uniform-excess-event} hold. Recall that we define
  $\femp{j} = \falg{j}{n_j}$, so that by the
  definition~\eqref{eqn:excess-event-1} and
  Assumption~\ref{assumption:budget} that $\femp{j}$ is a
  $\pen_j$-accurate minimizer of the empirical risk, we have
  \begin{equation}
    \label{eqn:emp-to-exp-fast-all}
    \risk(\femp{j}) \leq \risk(\fopt{j}) + 3\pen_j(n_j)
    + \const_2 \epsilon_j
  \end{equation}
  for any $j$.
  By our assumption that the index $j \le \iopt$, we have $\femp{j} \in
  \F_{\iopt}$, and since the event~\eqref{eqn:excess-event-2} holds for the
  classes $\iopt$ and $j$ (i.e.\ $\event{2}{\iopt j}{\epsilon_{\iopt}}$
  occurs), we further obtain that
  \begin{equation}
    \emprisk[\iopt](\femp{j}) - \emprisk[\iopt](\fopt{j}) \leq
    2\left(\risk(\femp{j}) - \risk(\fopt{j})\right) +
    \pen_{j}(n_{\iopt}) + \const_2 \epsilon_{\iopt}.
    \label{eqn:transfer-i-to-iopt}
  \end{equation}
  Applying the earlier bound~\eqref{eqn:emp-to-exp-fast-all}
  on $\risk(\femp{j}) - \risk(\fopt{j})$ to the
  inequality~\eqref{eqn:transfer-i-to-iopt}, we see that
  \begin{equation}
    \emprisk[\iopt](\femp{j}) - \emprisk[\iopt](\fopt{j})
    \le 6 \pen_j(n_j) + 2 \const_2 \epsilon_j + \pen_{j}(n_{\iopt})
    + \const_2 \epsilon_{\iopt}.
    \label{eqn:low-to-high-emp}
  \end{equation}
  Now we again use the fact that the event~\eqref{eqn:excess-event-2}
  holds so that $\event{2}{\iopt\iopt}{\epsilon_{\iopt}}$
  occurs. Using $f = \fopt{j}$ in the event since $\fopt{j} \in
  \F_{\iopt}$, we see that
  \begin{equation*}
    2\left(\risk(\fopt{j}) - \risk(\fopt{\iopt})\right) \ge
    \left(\emprisk[\iopt](\fopt{j}) -
    \emprisk[\iopt](\fopt{\iopt})\right) - \pen_{\iopt}(n_{\iopt}) -
    \const_2 \epsilon_{\iopt}.
  \end{equation*}
  Now apply the inequality~\eqref{eqn:low-to-high-emp} to
  lower bound $\emprisk[\iopt](\fopt{j})$ to see that
  \begin{align*}
    2\left(\risk(\fopt{j}) - \risk(\fopt{\iopt})\right)
    & \ge \emprisk[\iopt](\femp{j})
    - \emprisk[\iopt](\fopt{\iopt})
    - 6 \pen_j(n_j) - 2 \const_2 \epsilon_j -  \pen_j(n_{\iopt}) -
    \pen_{\iopt}(n_{\iopt}) - 2 \const_2 \epsilon_{\iopt} \\
    & \ge \emprisk[\iopt](\femp{j})
    - \emprisk[\iopt](\fopt{\iopt})
    - 6 \pen_j(n_j) - 2 \const_2 \epsilon_j - 2 \pen_{\iopt}(n_{\iopt})
    - 2 \const_2 \epsilon_{\iopt},
  \end{align*}
  where we have used the fact that $j \le \iopt$ so
  $\pen_{\iopt}(n_{\iopt}) \ge \pen_j(n_{\iopt})$.  Using the
  condition~\eqref{eqn:coarseestimate-fast-unspecified} that defines
  the selected index $\iopt$, we obtain
  \begin{align*}
    \lefteqn{2\left(\risk(\fopt{j}) - \risk(\fopt{\iopt})\right)} \\
    & \ge \emprisk[\iopt](\femp{\iopt})
    + \pconcunspec \pen_{\iopt}(n_{\iopt})
    + \pcondunspec \const_2 \epsilon_{\iopt}
    - \pconcunspec \pen_j(n_j)
    - \emprisk[\iopt](\fopt{\iopt})
    - 6 \pen_j(n_j) - 2 \const_2 \epsilon_j - 2 \pen_{\iopt}(n_{\iopt})
    - 2 \const_2 \epsilon_{\iopt} \nonumber \\
    & = \emprisk[\iopt](\femp{\iopt})
    - \emprisk[\iopt](\fopt{\iopt})
    + (\pconcunspec - 2) \pen_{\iopt}(n_{\iopt})
    - (6 + \pconcunspec)\pen_j(n_j)
    - 2 \const_2 \epsilon_j + (\pcondunspec - 2) \const_2 \epsilon_{\iopt}.
  \end{align*}
  Finally, we note that by the event~\eqref{eqn:excess-event-1}, since
  $\risk(\fopt{j}) - \risk(f) \le 0$ for all $f \in \F_j$, we have
  \begin{equation*}
    \emprisk[\iopt](\fopt{\iopt}) \le \emprisk[\iopt](\femp{\iopt}) +
    \half \pen_{\iopt}(n_{\iopt})
    + \half \const_2 \epsilon_{\iopt},
  \end{equation*}
  whence we obtain
  \begin{equation}
    2\left(\risk(\fopt{j}) - \risk(\fopt{\iopt})\right)
    \ge \left(\pconcunspec - 5/2\right) \pen_{\iopt}(n_{\iopt})
    - (6 + \pconcunspec) \pen_j(n_j) - 2 \const_2 \epsilon_j
    + (\pcondunspec - 5/2) \const_2 \epsilon_{\iopt}.
    \label{eqn:true-class-risk-diff}
  \end{equation}
  Applying the inequality~\eqref{eqn:emp-to-exp-fast-all} for
  the class $\iopt$, we have
  \begin{equation*}
    \risk(\fopt{j}) - \risk(\femp{\iopt})
    \ge \risk(\fopt{j}) - \risk(\fopt{\iopt})
    - 3 \pen_{\iopt}(n_{\iopt}) - \const_2 \epsilon_{\iopt},
  \end{equation*}
  and combining this inequality with the earlier
  guarantee~\eqref{eqn:true-class-risk-diff}, we find that
  \begin{equation*}
    2 \left(\risk(\fopt{j}) - \risk(\femp{\iopt})\right)
    \ge (\pconcunspec - 17/2) \pen_{\iopt}(n_{\iopt})
    - (6 + \pconcunspec) \pen_j(n_j) - 2 \const_2 \epsilon_j
    + (\pcondunspec - 9/2) \const_2 \epsilon_{\iopt}
  \end{equation*}
  Rearranging terms, we obtain the statement of the lemma.
\end{proof-of-lemma}

In order to prove Lemma~\ref{lemma:nesting-fast-upper}, we
need one more result:
\begin{lemma}
  Let the joint events~\eqref{eqn:uniform-excess-event}
  hold (i.e.\ $\unionevent{1}{\epsilon}$ and
  $\unionevent{2}{\epsilon}$). For $i, j \in
  \gridset_\penscalefac$ such that $i \geq j$ and
  \begin{equation*}
    \emprisk[i](\femp{j}) +
    \pconcunspec\pen_{j}(n_j) \leq
    \emprisk[i](\femp{i}) +
    \pconcunspec\pen_i(n_i) +
    \pcondunspec \const_2 \epsilon_i
  \end{equation*}
  we have
  \begin{equation*}
    \risk(\femp{j}) \le
    \risk(\fopt{i})
    + (2 \pconcunspec + 3) \pen_i(n_i)
    + (2 \pcondunspec + 1) \const_2 \epsilon_i.
  \end{equation*}
  \label{lemma:nesting-fast-upper-good}
\end{lemma}
\begin{proof}
  We begin by noting that since $i \ge j$, we have
  $\femp{j} \in \F_i$, and since the event~\eqref{eqn:excess-event-1}
  holds by assumption, we have
  \begin{equation*}
    \risk(\femp{j}) - \risk(\fopt{i})
    \le 2 \left(\emprisk[i](\femp{j}) - \emprisk[i](\fopt{i})\right)
    + \pen_i(n_i) + \const_2 \epsilon_i.
  \end{equation*}
  Recalling the inequality assumed in the condition of the lemma,
  we see that
  \begin{equation*}
    \risk(\femp{j}) - \risk(\fopt{i})
    \le 2\left(\emprisk[i](\femp{i})
    + \pconcunspec \pen_i(n_i) + \pcondunspec \const_2 \epsilon_i
    - \pconcunspec \pen_j(n_j) - \emprisk[i](\fopt{i})\right)
    + \pen_i(n_i) + \const_2 \epsilon_i.
  \end{equation*}
  Applying Assumption~\ref{assumption:budget}(d) on the empirical minimizers, we
  have $\emprisk[i](\femp{i}) - \emprisk[i](\fopt{i}) \le \pen_i(n_i)$, so
  \begin{equation*}
    \risk(\femp{j}) - \risk(\fopt{i})
    \le 2\left(
    (\pconcunspec + 1) \pen_i(n_i) + \pcondunspec \const_2 \epsilon_i
    - \pconcunspec \pen_j(n_j)
    \right)
    + \pen_i(n_i) + \const_2 \epsilon_i.
  \end{equation*}
  Ignoring the negative term $-\pconcunspec \pen_j(n_j)$ yields the lemma.
\end{proof}

\begin{proof-of-lemma}[\ref{lemma:nesting-fast-upper}]
  For $j \in \gridset_\penscalefac$, define $\pos{j}$ to be the position of
  class $j$ in the coarse-grid set (that is, $\pos{1} = 1$, the next class $j
  \in \gridset_\penscalefac$ has \mbox{$\pos{j} = 2$} and so on). We prove the
  lemma by induction on the class $j$ for $j \geq \iopt$, $j \in
  \gridset_\penscalefac$. Our inductive hypothesis is that
  \begin{equation}
    \risk(\femp{\iopt}) \leq \risk(\fopt{j}) +
    (\pos{j} - \pos{\iopt}+1)\left[(2 \pconcunspec + 3)\pen_j(n_j) +
    (2 \pcondunspec + 1) \const_2 \epsilon_j\right].
    \label{eqn:nesting-fast-induct}
  \end{equation}
  The base case for $j = \iopt$ is immediate since by assumption, the
  event~\eqref{eqn:excess-event-1} holds, so we obtain the
  inequality~\eqref{eqn:emp-to-exp-fast-all}.

  For the inductive step, we assume that the claim holds for all $\iopt \leq k
  \leq j-1$ such that $k \in \gridset_\penscalefac$ and establish the claim
  for $j$. Since $\iopt$ is the largest class in $\gridset_\penscalefac$
  satisfying the condition~\eqref{eqn:coarseestimate-fast-unspecified} and $j
  \geq \iopt$, there must exist a class $k < j$ in $\gridset_\penscalefac$ for
  which
  \begin{equation}
    \emprisk[j](\femp{k}) + \pconcunspec\pen_k(n_k)
    < \emprisk[j](\femp{j}) + \pconcunspec\pen_j(n_j)
    + \pcondunspec \const_2 \epsilon_j.
    \label{eqn:witness}
  \end{equation}
  By inspection, this is precisely the condition of
  Lemma~\ref{lemma:nesting-fast-upper-good}, so  
  \begin{equation*}
    \risk(\fopt{k})
    \leq \risk(\femp{k})
    \leq \risk(\fopt{j}) +
    (2 \pconcunspec + 3) \pen_j(n_j) + (2\pcondunspec + 1) \const_2 \epsilon_j.
  \end{equation*}
  Now there are two possibilities. If $k \leq \iopt$,
  Lemma~\ref{lemma:nesting-fast-lower} applies, and we recall the assumptions
  on $\pconcunspec$ and $\pcondunspec$, which guarantee $2 \pconcunspec + 3
  \ge 6 + \pconcunspec$ and $2 \pcondunspec + 1 \ge 2$. If $k \geq \iopt$,
  then we can apply our inductive hypothesis since $k < j$. In either case, we
  conclude that
  \begin{align*}
    \risk(\femp{\iopt})
    & \le \risk(\fopt{k}) + (\pos{k} -
    \pos{\iopt} + 1)\left[
      (2\pconcunspec + 3) \pen_k(n_k)
      + (2 \pcondunspec + 1) \const_2 \epsilon_k \right] \\
    & \le \risk(\fopt{k}) + (\pos{j} - 1 -
    \pos{\iopt}+1) \left[(2 \pconcunspec + 3) \pen_j(n_j)
      + (2 \pcondunspec + 1) \const_2 \epsilon_j\right],
  \end{align*}
  where the final inequality uses $\pos{k} \leq \pos{j}-1$ and the
  monotonicity assumptions~\ref{assumption:budget}(a)-(b).  Applying
  the relationship~\eqref{eqn:witness} of the risk of $\fopt{k}$ to
  that of $\fopt{j}$ shows that the inductive
  hypothesis~\eqref{eqn:nesting-fast-induct} holds at $i$. Noting that
  $\gridsize \ge \pos{j} - \pos{\iopt} + 1$ completes the proof.
\end{proof-of-lemma}


\section{Proof of Lemma~\ref{lemma:ucbbadevent}}
\label{appendix:proof-ucb-bad-event}

Following~\cite{AuerCBF02}, we show that the event in the lemma occurs
with very low probability by breaking it up into smaller events more
amenable to analysis. Recall that we are interested in controlling the
probability of the event
\begin{equation}
\obj(i, n_i s_i) - \const_2 \sqrt{\frac{\log T}{n_is_i}}
\le \obj(\opt, n_\opt s_\opt) - \const_2 \sqrt{\frac{\log
    T}{n_{\opt}s_{\opt}}}
\label{eqn:ucbbadevent}
\end{equation}
For this bad event to happen, at least one of the following
three events must happen:
\begin{subequations}
  \begin{align}
    & \emprisk[n_is_i](\falg{i}{n_i s_i}) - \inf_{f \in \F_i}\risk(f)
    \leq -\pen_i(n_is_i)
    - \const_2 \sqrt{\frac{\log K}{n_i s_i}} - \const_2 \sqrt{\frac{\log T}{n_i
        s_i}} \label{eqn:bad-event1} \\ 
    & \emprisk[n_{\opt}s_{\opt}](\falg{\opt}{n_{\opt} s_{\opt}}) -
    \inf_{f \in \F_{\opt}} \risk(f)
    \ge \pen_i(n_{\opt}s_{\opt}) + \const_2
    \sqrt{\frac{\log K}{n_{\opt}s_{\opt}}}
    + \const_2 \sqrt{\frac{\log T}{n_{\opt}s_{\opt}}}
    \label{eqn:bad-event2}\\
    & \risk[i]^* + \pen_i(T n_i)
    \leq \risk^* + \pen_\opt(T n_\opt)
    + 2\left(\pen_i(n_is_i) + \const_2 \sqrt{\frac{\log K}{n_is_i}}
    + \const_2 \sqrt{\frac{\log T}{n_is_i}}\right)\label{eqn:bad-event3}.
  \end{align}
\end{subequations}
Temporarily use the shorthand $f_i = \falg{i}{n_i s_i}$ and $f_{\opt}
= \falg{\opt}{n_{\opt} s_{\opt}}$.  The relationship between
Eqs.~(\ref{eqn:bad-event1})--(\ref{eqn:bad-event3}) and the event in
\eqref{eqn:ucbbadevent} follows from the fact that if none
of \eqref{eqn:bad-event1}--\eqref{eqn:bad-event3} occur, then
\begin{align*}
  \obj(i, n_i s_i) -
    \const_2 \sqrt{\frac{\log T}{n_is_i}}
  & ~ = \emprisk[n_i s_i](f_i)
  + \pen_i(T n_i)
  - \pen_i(n_i s_i) - \const_2 \sqrt{\frac{\log K}{n_i s_i}}
  - \const_2 \sqrt{\frac{\log T}{n_is_i}} \\
  & \stackrel{(\ref{eqn:bad-event1})}{>}
  \inf_{f \in \F_i} \risk(f) + \pen_i(\timesteps n_i)
  - 2\left(\pen_i(n_is_i) + \const_2 \sqrt{\frac{\log
      K}{n_is_i}} + \const_2 \sqrt{\frac{\log
      t}{n_is_i}}\right) \\
  & \stackrel{(\ref{eqn:bad-event3})}{>}
  \inf_{f \in \F_{\opt}} \risk(f) + \pen_\opt(\timesteps n_\opt)
  + 2\left(\pen_i(n_is_i)
  + \const_2 \sqrt{\frac{\log K}{n_is_i}}
  + \const_2 \sqrt{\frac{\log T}{n_is_i}}\right) \\*
  & \qquad\qquad { }
  - 2\left(\pen_i(n_is_i) + \const_2 \sqrt{\frac{\log K}{n_is_i}}
  + \const_2 \sqrt{\frac{\log
      n}{n_is_i}}\right) \\
  & \stackrel{(\ref{eqn:bad-event2})}{>}
  \emprisk[n_{\opt} s_{\opt}](f_{\opt}) + \pen_\opt(\timesteps n_\opt)
  - \pen_i(n_{\opt}s_{\opt}) -
  \const_2 \sqrt{\frac{\log K}{n_{\opt}s_{\opt}}} - \const_2 \sqrt{\frac{\log
      t}{n_{\opt}s_{\opt}}} \\
  & ~ = \obj(\opt, n_\opt s_\opt) - \const_2 \sqrt{\frac{\log
      t}{n_{\opt}s_{\opt}}}. 
\end{align*}
From the above string of inequalities, to show that the
event~\eqref{eqn:ucbbadevent} has low probability, we need simply
show that each of \eqref{eqn:bad-event1}, \eqref{eqn:bad-event2}, and
\eqref{eqn:bad-event3} have low probability.

To prove that each of the bad events have low probability, we note the
following consequences of Assumption~\ref{assumption:uniformity}.  Recall
the definition of $f_i^*$ as the minimizer of $\risk(f)$ over the class
$\F_i$. Then by
Assumption~\ref{assumption:uniformity}(\ref{assumption:uniformity-alg-bandits}),
\begin{equation*}
  \risk(f_i^*) - \pen_i(n) - \const_2 \epsilon
  \le \risk(\falg{i}{n}) - \pen_i(n) - \const_2 \epsilon
  < \emprisk[n](\falg{i}{n}),
\end{equation*}
while
Assumptions~\ref{assumption:uniformity}(\ref{assumption:uniformity-min-bandits})
and~\ref{assumption:uniformity}(\ref{assumption:uniformity-fixed-bandits}) imply
\begin{equation*}
  \emprisk[n](\falg{i}{n}) \le \emprisk[n](f_i^*) + \pen_i(n)
  \le \risk(f_i^*) + \pen_i(n) + \const_2 \epsilon,
\end{equation*}
each with probability at least $1 - \const_1 \exp(-4n\epsilon^2)$. In
particular, we see that
the events \eqref{eqn:bad-event1} and \eqref{eqn:bad-event2} have low
probability:
\begin{align*}
  \lefteqn{\P\left[
      \emprisk[n_is_i](\falg{i}{n_i s_i}) - \risk(f_i^*) \leq -\pen_i(n_is_i)
      - \const_2 \sqrt{\frac{\log K}{n_i s_i}}
      - \const_2 \sqrt{\frac{\log T}{n_i s_i}}
      \right]} \\
  & \quad\quad\quad
  \le \const_1 \exp\left(-4 n_is_i\left(\frac{\log K}{n_is_i} + \frac{\log
    t}{n_is_i}\right)\right) = \frac{\const_1}{(tK)^4} \\
  \lefteqn{\P\left[
      \emprisk[n_\opt s_\opt](\falg{\opt}{n_{\opt} s_{\opt}}) -
      \risk^* \ge \pen_\opt(n_\opt s_\opt)
      + \const_2 \sqrt{\frac{\log K}{n_\opt s_\opt}} +
      \const_2 \sqrt{\frac{\log T}{n_\opt s_\opt}}\right]} \\
  & \quad\quad\quad
  \leq \const_1\exp\left(-4 n_\opt s_\opt \left(\frac{\log K}{n_\opt s_\opt}
  + \frac{\log T}{n_\opt s_\opt }\right)\right) = \frac{\const_1}{(tK)^{4}}.
\end{align*}

What remains is to show that for large enough $\tau$, \eqref{eqn:bad-event3}
does not happen.  Recalling the definition that $\risk^* + \pen_{\opt}(T
n_{\opt}) = \risk[i]^* + \pen_i(T n_i) - \excess_i$, we
see that for \eqref{eqn:bad-event3} to fail it is sufficient that
\begin{equation*}
  \excess_i > 2 \pen_i(\tau n_i) + 2 \const_2 \sqrt{\frac{\log K}{n_i \tau}}
  + 2 \const_2 \sqrt{\frac{\log T}{n_i \tau}}.
\end{equation*}
Let $x \wedge y \defeq \min\{x, y\}$ and $x \vee y \defeq \max\{x, y\}$.
Since $\pen_i(n) \le \penconstant_i n^{-\alpha_i}$,
the above is satisfied when
\begin{equation}
  \label{eqn:sufficiently-large-tau}
  \frac{\excess_i}{2} > \penconstant_i (\tau
  n_i)^{-(\alpha_i \wedge \half)}
  + \const_2 \sqrt{\log K} (\tau n_i)^{-(\alpha_i \wedge \half)}
  + \const_2 \sqrt{\log T} (\tau n_i)^{-(\alpha_i \wedge \half)}
\end{equation}
We can solve \eqref{eqn:sufficiently-large-tau} above and see immediately
that if
\begin{equation*}
  \tau_i > \frac{2^{1/\alpha_i \vee 2}
    (\penconstant_i + \const_2 \sqrt{\log T} +
    \const_2 \sqrt{\log K})^{1/\alpha_i \vee 2}}{
    n_i\excess_i^{1/\alpha_i \vee 2}},
\end{equation*}
then
\begin{equation}
  \risk[i]^* > \risk^* + 2\left(\pen_i(n_i \tau_i)
  + \const_2 \sqrt{\frac{\log K}{n_i \tau_i}}
  + \const_2 \sqrt{\frac{\log T}{n_i \tau_i}}\right).
  \label{eqn:bound-tau-i}
\end{equation}
Thus the event in~\eqref{eqn:bad-event3} fails to occur, completing
the proof of the lemma.


\section{Proofs of Proposition~\ref{proposition:regret-bound} and Theorem~\ref{theorem:no-separation}}
\label{sec:no-separation}

In this section we provide proofs for
Proposition~\ref{proposition:regret-bound} and
Theorem~\ref{theorem:no-separation}. The proof of the proposition
follows by dividing the model clases into two groups: those for which
$\excess_i > \gamma$, and those with small excess risk,
i.e.\ $\excess_i < \gamma$. Theorem~\ref{theorem:expected-pulls}
provides an upper bound on the fraction of budget allocated to model
classes of the first type. For the model classes with small excess
risk, all of them are nearly as good as $\opt$ in the regret criterion
of Proposition~\ref{proposition:regret-bound}. Combining the two
arguments gives us the desired result.

Of course, the proposition has the drawback that it does not provide
us with a prescription to select a good model or even a model
class. This shortcoming is addressed by
Theorem~\ref{theorem:no-separation}. The theorem relies on an
averaging argument used quite frequently to extract a good solution
out of online learning or stochastic optimization
algorithms~\cite{CesaBianchiCoGe04,NemirovskiJuLaSh09}.

\subsection{Proof of Proposition~\ref{proposition:regret-bound}}

Define $\beta_i = \max\{1/\alpha_i, 2\}$ as in the conclusion of
Theorem~\ref{theorem:expected-pulls}, and let $b_i = \penconstant_i +
\const_2 \sqrt{\log T}$.  Dividing the regret into classes with high
and low excess penalized risk $\excess_i$, for any threshold $\gamma
\ge 0$ we have by a union bound that with probability at least $1 -
\const_1 / T K^3$,
\begin{align*}
  \sum_{i=1}^K \excess_i T_i(T)
  & = \sum_{\{i \mid \excess_i \ge \gamma\}}
  \excess_iT_i(T) + \sum_{\{i \mid \excess _i \leq \pen\}} \excess_iT_i(T) \\
  & \le C \sum_{\{i \mid \excess_i \geq \gamma\}} \excess_i
  \frac{b_i^{\beta_i}}{
    n_i\excess_i^{\beta_i}} + \gamma T
  ~ \le ~ C \sum_{i=1}^K
  \frac{b_i^{\beta_i}}{
    n_i\gamma ^{\beta_i - 1}} + \gamma T.
\end{align*}

To simplify this further, we use the assumption that $\alpha_i \equiv
\alpha$ for all $i$. Hence the complexity penalties of the classes
differ only in the sampling rates $n_i$, that is,
\begin{equation}
  \sum_{i=1}^K \excess_i T_i(\timesteps) \le
  \frac{1}{\gamma^{\beta - 1}}
  \sum_{i=1}^K \frac{C b_i^{\beta_i}}{n_i} + \gamma \timesteps.
  \label{eqn:regret-to-minimize}
\end{equation}
Minimizing the bound~\eqref{eqn:regret-to-minimize} over $\gamma$ by
taking derivatives, we get
\begin{equation*}
  \gamma = T^{-\frac{1}{\beta}} (\beta - 1)^{\frac{1}{\beta}}
  \left(\sum_{i=1}^K \frac{C b_i^\beta}{n_i}\right)^{\frac{1}{\beta}},
\end{equation*}
which, when plugged back into \eqref{eqn:regret-to-minimize}, gives
\begin{equation*}
  \sum_{i=1}^K \excess_i T_i(\timesteps)
  \le 2
  \left(\sum_{i=1}^K \frac{C b_i^\beta}{n_i}\right)^{1/\beta}
  (\beta - 1)^{1/\beta} \timesteps^{1 - 1/\beta}.
\end{equation*}
Noting that $\frac{1}{\beta}\log(\beta - 1) \le \frac{\beta -
  2}{\beta} < 1$, we see that $(\beta - 1)^{1/\beta} <
\exp(1)$. Plugging the definition of $\beta = \max\{1/\alpha, 2\}$, so
that $1/\beta = \min\{\alpha, \half\}$, gives the result of the
proposition.

\subsection{Proof of Theorem~\ref{theorem:no-separation}}

Before proving the theorem, we state a technical lemma that makes our
argument somewhat simpler.
\begin{lemma}
  \label{lemma:maximize-selection}
  For $0 < p < 1$ and $a \succ 0$, consider the optimization problem
  \begin{equation*}
    \max_x ~ \sum_{i=1}^K a_i x_i^p
    ~~~ {\rm s.t.} ~~~
    \sum_{i=1}^K x_i \le T, ~ x_i \ge 0.
  \end{equation*}
  The solution of the problem is to take $x_i \propto a_i^{1 / (1 - p)}$,
  and the optimal value is
  \begin{equation*}
    T^p \left(\sum_{i=1}^K a_i^{\frac{1}{1 - p}}\right)^{1 - p}.
  \end{equation*}
\end{lemma}
\begin{proof}
Reformulating the problem to make it a minimization problem, that is,
our objective is $-\sum_{i=1}^K a_i x_i^p$, we have a convex problem.
Introducing Lagrange multipliers $\theta \ge 0$ and $\nu \in \R_+^K$
for the inequality constraints, we have Lagrangian
\begin{equation*}
  \mc{L}(x, \theta, \nu) =
  -\sum_{i=1}^K a_i x_i^p + \theta\left(\sum_{i=1}^K x_i - T\right)
  - \<\nu, x\>.
\end{equation*}
To find the infimum of the Lagrangian over $x$, we take derivatives
and see that $-a_i p x_i^{p - 1} + \theta - \nu_i = 0$, or that $x_i =
a_i^{-1 / (p - 1)} p^{-1 / (p - 1)}(\theta - \nu_i)^{1 / (p - 1)}$.
Since $a_i > 0$, the complimentary slackness conditions for $\nu$ are
satisfied with $\nu = 0$, and we see that $\theta$ is simply a
multiplier to force the sum $\sum_{i=1}^K x_i = T$. That is, $x_i
\propto a_i^{1 / (1 - p)}$, and normalizing appropriately, $x_i = T
a_i^{1 / (1 - p)} / \sum_{j=1}^K a_j^{1 / (1 - p)}$.  By plugging
$x_i$ into the objective, we have
\begin{equation*}
  \sum_{i=1}^K a_i x_i^p
  = T^p \frac{\sum_{i=1}^K a_i a_i^{p / (1 - p)}}{\left(\sum_{j=1}^K
    a_j^{1 / (1 - p)}\right)^p}
  = T^p \frac{\sum_{i=1}^K a_i^{1 / (1 - p)}}{\left(\sum_{j=1}^K
    a_j^{1 / (1 - p)}\right)^p}
  = T^p \left(\sum_{i=1}^K a_i^{1 / (1 - p)}\right)^{1 - p}
  \qedhere
\end{equation*}
\end{proof}

With the Lemma~\ref{lemma:maximize-selection} in hand, we proceed with the
proof of Theorem~\ref{theorem:no-separation}.  As before, we use the
shorthand $\beta = \max\{1/\alpha, 2\}$ throughout the proof to reduce
clutter. We also let $\numpulls_i(t)$ be the number of times class $i$ was
selected by time $t$.  Recalling the definition of the regret
from~\eqref{eqn:excess} and the result of the previous proposition, we have
with probability at least $1 - \const_1 / (TK^3)$
\begin{equation*}
  \frac{1}{T} \sum_{t=1}^T[R_{\isubt}^* + \pen_{\isubt}(Tn_{\isubt})]
  \leq R^* + \pen_{i^*}(Tn_{i^*}) + 2 e \const_2
  T^{- 1/\beta}\sqrt{\log T}\left(
  \sum_{i=1}^K\frac{C}{n_i}\right)^{1/\beta}.
\end{equation*}
Using the definition of $f_i^*$ as the minimizer of $\risk(f)$ over
$\F_i$, we use
Assumptions~\ref{assumption:uniformity}(\ref{assumption:uniformity-min-bandits})
and \ref{assumption:uniformity}(\ref{assumption:uniformity-fixed-bandits}) to
see that for fixed $\numpulls_i$, with probability at least $1 -
\const_1 / (TK)^4$,
\begin{equation}
  \emprisk[n_i \numpulls_i](\falg{i}{n_i \numpulls_i})
  \le \emprisk[n_i \numpulls_i](f_i^*) + \pen_i(n_i \numpulls_i)
  \le \risk(f_i^*) + \pen_i(n_i \numpulls_i)
  + \const_2 \sqrt{\frac{\log K}{n_i \numpulls_i}} + \const_2
  \sqrt{\frac{\log T}{n_i \numpulls_i}}.
  \label{eqn:high-prob-play}
\end{equation}
Denote by $f_t$ the output of $\alg$ on round $t$, that is, $f_t =
\falg{i_t}{n_{i_t} \numpulls_{i_t}(t)}$. By the previous
equation~\eqref{eqn:high-prob-play}, we can use a union bound and the
regret bound from Proposition~\ref{proposition:regret-bound} to
conclude that with probability at least $1 - \const_1 / (TK^3) -
\const_1 / (T^3K^3)$,
\begin{align}
  \lefteqn{\frac{1}{T}\sum_{t=1}^T \emprisk[n_{i_t} \numpulls_{i_t}(t)](f_t)
    + \pen_{i_t}(T n_{i_t})} \nonumber \\
  & \le \frac{1}{T} \sum_{t=1}^T \left[
    \gamma_i(n_{i_t} \numpulls_{i_t}(t))
    + \const_2 \sqrt{\frac{\log K}{n_{i_t} \numpulls_{i_t}(t)}}
    + \const_2 \sqrt{\frac{\log T}{n_{i_t} \numpulls_{i_t}(t)}}\right]
  + \frac{1}{T} \sum_{t=1}^T \left[
    \risk[i_t]^* + \pen_{i_t}(T n_{i_t}) \right] \nonumber \\
  & \le \frac{1}{T} \sum_{t=1}^T \left[
    \gamma_i(n_{i_t} \numpulls_{i_t}(t))
    + \const_2 \sqrt{\frac{\log K}{n_{i_t} \numpulls_{i_t}(t)}}
    + \const_2 \sqrt{\frac{\log T}{n_{i_t}
        \numpulls_{i_t}(t)}}\right] + \risk(f_i^*) + \pen_i(n_i
  \numpulls_i) \nonumber\\
  &\qquad\qquad\qquad\qquad\qquad\qquad\qquad\qquad\qquad\qquad+ 2 e
  \const_2 T^{-1/\beta} \sqrt{\log T} 
  \left(\sum_{i=1}^K \frac{C}{n_i}\right)^{1/\beta}.
  \label{eqn:emprisk-regret}
\end{align}

Now we again make use of
Assumption~\ref{assumption:uniformity}(\ref{assumption:uniformity-alg-bandits})
to note that with probability at least $1 - \const_1 / (T^4 K^4)$,
\begin{equation*}
  \risk(f_t)
  \le \emprisk[n_{i_t} \numpulls_{i_t}(t)](f_t)
  + \pen_{i_t}(n_{i_t} \numpulls_{i_t}(t))
  + \const_2 \sqrt{\frac{\log K}{n_{i_t} \numpulls_{i_t}(t)}}
  + \const_2 \sqrt{\frac{\log T}{n_{i_t} \numpulls_{i_t}(t)}}.
\end{equation*}
Using a union bound and applying the empirical risk bound
\eqref{eqn:emprisk-regret}, we drop the positive $\pen_{i_t}(T
n_{i_t})$ terms from the left side of the bound and see that with
probability at least $1 - \const_1 / (T K ^3) - 2 \const_1 / (T^3
K^3)$,
\begin{align}
  \frac{1}{T} \sum_{t=1}^T \risk(f_t)
  & \le
  \risk^* + \pen_{\opt}(T n_\opt)
  + 2 e \const_2 T^{-1/\beta} \sqrt{\log T} \left(\sum_{i=1}^K
  \frac{C}{n_i}\right)^{1/\beta} \nonumber \\
  & \quad ~ + \frac{2}{T} \sum_{t=1}^T \left[
    \gamma_i(n_{i_t} \numpulls_{i_t}(t))
    + \const_2 \sqrt{\frac{\log K}{n_{i_t} \numpulls_{i_t}(t)}}
    + \const_2 \sqrt{\frac{\log T}{n_{i_t} \numpulls_{i_t}(t)}}\right].
  \label{eqn:risk-regret}
\end{align}

Defining $\favg \defeq \frac{1}{T} \sum_{t=1}^T f_t$, we use Jensen's
inequality to see that $\risk(\favg) \le \frac{1}{T} \sum_{t=1}^T
\risk(f_t)$. Thus, all that remains is to control the last sum in
\eqref{eqn:risk-regret}. Using the definition of $\pen_i$, we replace
the sum with
\begin{align*}
  \lefteqn{\sum_{t=1}^T c_i n_{i_t}^{-\alpha} \numpulls_{i_t}(t)^{-\alpha}
    + n_{i_t}^{-\half} \numpulls_{i_t}(t)^{-\half} \const_2
    \left[\sqrt{\log K} + \sqrt{\log T}\right]} \\
  & \le \sum_{t=1}^T \left[c_i n_{i_t}^{-\alpha}
    + \const_2 n_{i_t}^{-\half} \sqrt{\log K} +
    \const_2 n_{i_t}^{-\half} \sqrt{\log T}\right]
  \numpulls_{i_t}(t)^{-\min\{\alpha, \half\}}.
\end{align*}
Noting that
\begin{equation*}
  \sum_{t : i_t = i} s_{i_t}(t)^{-\min\{\alpha, \half\}}
  = \sum_{t = 1}^{T_i(T)} t^{-1/\beta}
  \le C' \, T_i(T)^{1 - 1/\beta}
\end{equation*}
for some constant $C'$ dependent on $\alpha$, we can upper bound the
last sum in \eqref{eqn:risk-regret} by
\begin{align}
  \lefteqn{\sum_{t=1}^T \left[
      \pen_i(n_{i_t} \numpulls_{i_t}(t))
      + \const_2 \sqrt{\frac{\log K}{n_{i_t} \numpulls_{i_t}(t)}}
      + \const_2 \sqrt{\frac{\log T}{n_{i_t} \numpulls_{i_t}(t)}}\right]}
  \nonumber \\
  & \le C' \sum_{i=1}^K
  \left[c_i n_i^{-\alpha}
    + \const_2 n_i^{-\half} \sqrt{\log K} +
    \const_2 n_i^{-\half} \sqrt{\log T}\right]
  T_i(T)^{1 - 1/\beta}.
  \label{eqn:num-selects-to-bound}
\end{align}
Now that we have a sum of order $K$ with terms $T_i(T)$ that are
bounded by $T$, that is, $\sum_{i=1}^K T_i(T) = K$, we can apply
Lemma~\ref{lemma:maximize-selection}. Indeed, we set $p = 1 - 1/\beta
= 1 - \min\{\alpha, \half\}$ and $a_i = c_i n_i^{-\alpha} + \const_2
n_i^{-\half}[\sqrt{\log K} + \sqrt{\log T}]$ in the lemma, and we see
immediately that \eqref{eqn:num-selects-to-bound} is upper bounded by
\begin{equation*}
  C' \, T^{1 - \min\{\alpha, \half\}} \left(\sum_{i=1}^K
  \left[c_i n_i^{-\alpha}
    + \const_2 n_i^{-\half} \sqrt{\log K} +
    \const_2 n_i^{-\half} \sqrt{\log T}\right]^{\max\{1/\alpha, 2\}}
  \right)^{\min\{\alpha, \half\}}.
\end{equation*}
Dividing by $T$ completes the proof that the average $\favg$ has good
risk properties with probability at least $1 - \const_1/(TK^3) -
2\const_1(T^3 K^3) > 1 - 2\const_1 / (TK^3)$.

\bibliographystyle{abbrvnat}
\bibliography{bib}

\end{document}